\ifdefined\pdfobjcompresslevel
\fi
\PassOptionsToPackage{table}{xcolor}
\documentclass[manuscript,screen,nonacm]{acmart}

\setcopyright{none}
\renewcommand{\footnotetextcopyrightpermission}[1]{}
\usepackage[utf8]{inputenc}
\usepackage[T1]{fontenc}
\usepackage{amsmath,amsthm,bm}
\usepackage{booktabs,tabularx,array,longtable,ragged2e,multirow,makecell}
\usepackage{graphicx}
\usepackage{nicefrac}
\usepackage{microtype}
\usepackage{enumitem}
\usepackage{xurl}
\usepackage{float}
\usepackage{listings}
\usepackage{placeins}
\usepackage{needspace}
\usepackage{pifont}

\hypersetup{
  colorlinks=false,
  pdfauthor={Jian Chen and Zixuan Yuan},
  pdftitle={When Evidence Changes the Subject: Subject-Typed Claim Licensing for Learned Routing}
}

\newcommand{\suite}{\textsc{SCOPE-Routing}}
\newcommand{\claims}{\textsc{ClaimScope-120}}
\newcommand{\real}{\textsc{RouteScope-Lite}}

\newcommand{\DataSP}{\textsc{DataSP}}
\newcommand{\FrontierCO}{\textsc{FrontierCO}}

\newcommand{\promoted}{\textsc{supported}}
\newcommand{\narrowed}{\textsc{narrowed}}
\newcommand{\rejected}{\textsc{unsupported}}
\newcommand{\diagnostic}{\textsc{diagnostic}}
\newcommand{\offthesis}{\textsc{out-of-contract}}

\newcommand{\tighttable}{\setlength{\tabcolsep}{3.6pt}\renewcommand{\arraystretch}{1.12}}
\newcommand{\verytighttable}{\setlength{\tabcolsep}{2.5pt}\renewcommand{\arraystretch}{1.06}}

\newcommand{\SubmissionSnapshot}{\texttt{core\_snapshot\_20260505}}
\newcommand{\ExtensionSnapshot}{\texttt{extension\_snapshot\_20260728}}

\definecolor{promotecolor}{RGB}{220,242,214}
\definecolor{narrowcolor}{RGB}{255,244,204}
\definecolor{rejectcolor}{RGB}{252,220,220}
\definecolor{lightgray}{RGB}{246,246,246}

\newcolumntype{Y}{>{\RaggedRight\arraybackslash}X}
\newcolumntype{C}[1]{>{\Centering\arraybackslash}p{#1}}
\newcolumntype{M}[1]{>{\Centering\arraybackslash}p{#1}}
\newcolumntype{P}[1]{>{\RaggedRight\arraybackslash}p{#1}}
\newcolumntype{R}[1]{>{\RaggedLeft\arraybackslash}p{#1}}

\newtheorem{proposition}{Proposition}

\title[Subject-Typed Claim Licensing for Learned Routing]{When Evidence Changes the Subject: Subject-Typed Claim Licensing for Learned Routing}

\author{Jian Chen}
\orcid{1234-5678-9012}
\affiliation{%
  \institution{The Hong Kong University of Science and Technology (Guangzhou)}
  \city{Guangzhou}
  \country{China}
}
\email{jchen524@connect.hkust-gz.edu.cn}

\author{Zixuan Yuan}
\authornote{Corresponding author.}
\affiliation{%
  \institution{The Hong Kong University of Science and Technology (Guangzhou)}
  \city{Guangzhou}
  \country{China}
}
\email{zixuanyuan@hkust-gz.edu.cn}

\begin{document}

\begin{abstract}
Modern learned systems increasingly combine learned components with search, repair, or external solvers. Benchmarks often measure the resulting end-to-end system, while scientific claims may concern only one component, creating an attribution problem: evidence can fail to support the requested component-level claim while still supporting a positive conclusion about the larger system. Existing evidence-to-claim methods primarily calibrate claim strength. We argue that composite systems require a second dimension: \emph{scientific subject}. We address this problem with \emph{subject-typed claim licensing}, which separates weaker conclusions about the requested subject from positive but non-substitutive credit about another subject. We instantiate this idea in \suite{} for preference-conditioned multigraph routing. Non-authors reproducibly apply the declared semantics; held-out review yields fewer reference-relative upward deviations than unstructured review, while the difference from a strong evidence checklist remains unresolved; and a controlled routing study shows that score-optimal and claim-eligible methods can differ while valid hybrid-system credit is preserved.
These results motivate treating claim strength and scientific subject as distinct dimensions of evidence-based evaluation.
\end{abstract}


\maketitle

\section{Introduction}
\label{sec:intro}

Modern learned systems are increasingly composite: a learned component may operate with search or repair \citep{li2025destroyrepair}, retrieval-augmented or modular language-processing pipelines \citep{chen2024fintextqa,chen2025dekeynlu}, or downstream experts \citep{ghoummaid2024act,mao2024multiexpert}, while benchmark evaluation measures the behavior of the complete executed system. This creates a scientific attribution problem. As Figure~\ref{fig:main} illustrates, two routers can achieve similar end-to-end regret even when one solves requests through its learned policy and the other relies heavily on a backup solver. The scores may be similar, but the evidence supports claims about different subjects.

This distinction creates two forms of \emph{partial support}. Evidence may support a weaker conclusion about the \emph{same scientific subject}---for example, static rather than dynamically robust learned routing \citep{xiao2024robustrouting,heakl2025svrpbench}. Or it may support a positive conclusion about a \emph{different subject}---for example, the backup-enabled system rather than the learned policy. The first changes claim strength; the second changes what the supported claim is about. Collapsing the two can either misattribute system-level evidence to a component or discard genuine positive evidence about the system.

\begin{figure}[t]
    \centering
    \includegraphics[width=\textwidth]{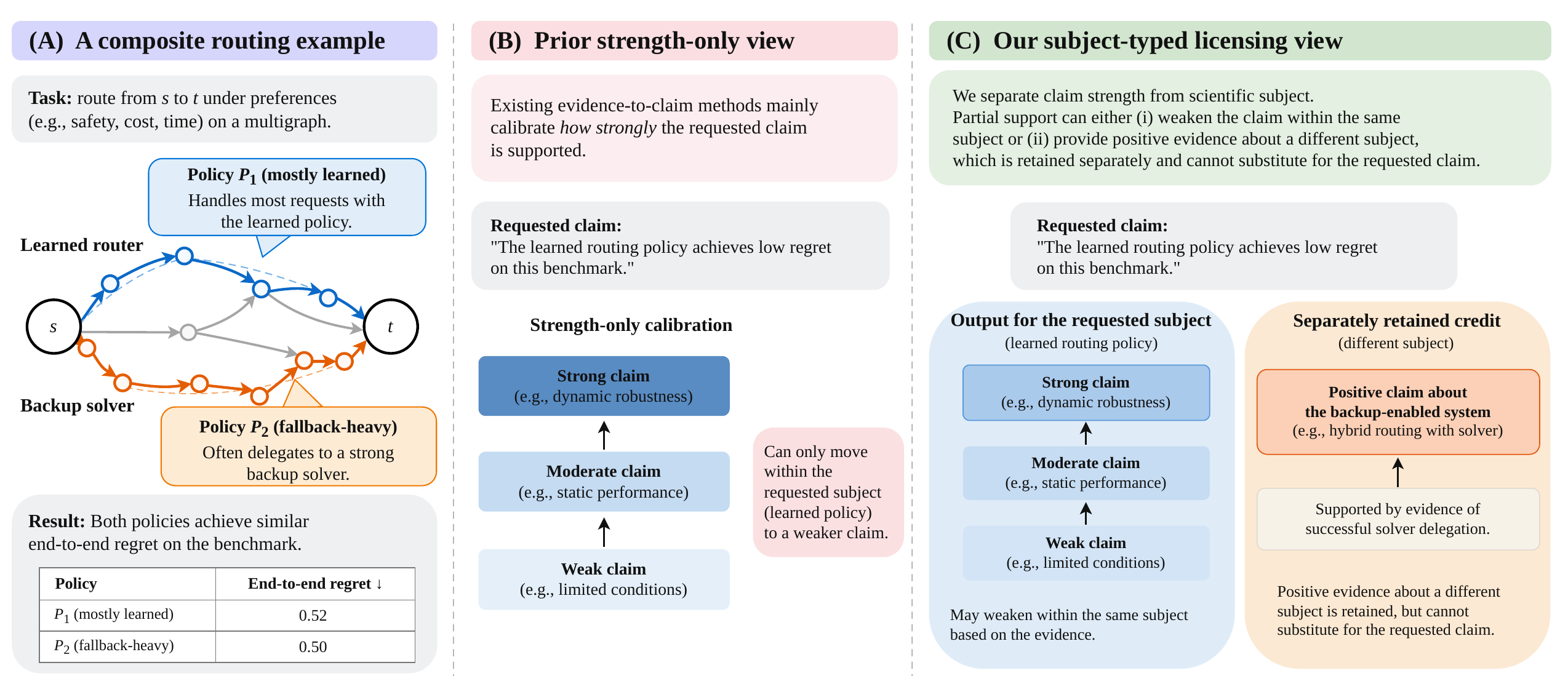}
    \Description{A three-panel comparison of a composite learned-routing example, a strength-only claim view, and subject-typed licensing with separate requested-subject support and retained hybrid-system credit.}
    \caption{When evidence changes the subject. A fallback-heavy router illustrates the difference between strength-only calibration and subject-typed licensing. Strength-only calibration can only move along the requested-subject ladder; subject-typed licensing additionally retains positive evidence about the backup-enabled hybrid without allowing that evidence to support the standalone learned-router claim.}
    \label{fig:main}
\end{figure}

Recent work has made evidence-to-claim reasoning increasingly explicit. Benchmark-assessment and metadata standards expose evaluation conditions, lifecycle, and scope \citep{reuel2024betterbench,akhtar2024croissant}; recent claim-verification methods assess whether designated statements are grounded in evidence, including claims drawn from scientific literature \citep{tang2024minicheck,jansen2025matteroffact}; and evidence-licensed calibration and claim replay formalize how evidence constrains the strength of reportable claims \citep{li2026calibration}. These approaches clarify \emph{how strongly} a claim is supported. Composite systems expose an additional representational question: \emph{which scientific subject is supported?} When positive evidence concerns another component or the executed system, it need not lie on the weaker-to-stronger scale of the originally requested component claim.

We address this problem with \emph{subject-typed claim licensing}. The central idea is to separate claim strength from scientific subject. A requested claim may weaken only within the same subject, while positive evidence about a related but incomparable subject is retained separately and cannot substitute for evidence about the requested subject. We instantiate this semantics in \suite{} for preference-conditioned multigraph routing. Routing is a useful testbed because repair, multi-expert solver composition, edge identity, preference coverage, and dynamic response make attribution and claim scope directly observable \citep{li2025destroyrepair,pan2025moses,rydin2025beyond,chen2025rethinking,xiao2024robustrouting,heakl2025svrpbench}.

We evaluate the approach along the path from scientific text to scientific decision. First, we test whether non-authors can reconstruct the required claim--evidence objects and apply a frozen contract consistently. Second, we compare held-out expert judgments under careful review, a strong evidence checklist, and subject-typed licensing. Third, we compare score-only method selection with selection conditioned on a requested scientific claim. Non-authors reproducibly apply the declared semantics; subject-typed review yields fewer reference-relative upward deviations than unstructured review, while its difference from the checklist remains unresolved; and the controlled routing study shows that score minimization and target-claim eligibility can select different policies while preserving valid hybrid-system credit.

\paragraph{Contributions.}
Conceptually, we identify scientific subject as a distinct dimension of evidence-relative claim interpretation. Formally, we introduce a subject-typed licensing semantics that separates same-subject narrowing from non-substitutive cross-subject credit. Empirically, we study its reconstruction, execution, reviewer consequences, and claim-conditioned selection in learned routing.

\section{Related Work}
\label{sec:related_work}
\label{sec:positioning}

\paragraph{Evidence records, validity, and claim support.}
Evaluation methodology first determines what evidence is available and what interpretations that evidence can sustain. Decision-quality metrics and stress tests characterize observed behavior, while datasheets, model cards, BenchmarkCards, behavioral evaluation, benchmark-assessment frameworks, and reproducibility standards document conditions, intended uses, artifacts, and failure modes \citep{gebru2021datasheets,mitchell2019model,ribeiro2020beyond,pineau2021improving,sokol2024benchmarkcards,reuel2024betterbench}. Validity-centered approaches ask which interpretations of observed measurements are warranted, and structured assurance approaches make claims, assumptions, evidence, and support relations explicit \citep{kane2013validating,rhodes2010software,freiesleben2025benchmarking,salaudeen2025measurement,bean2025measuring}. Recent claim-assessment systems additionally retrieve evidence, evaluate claim--evidence relations, or estimate support and overstatement for designated scientific statements \citep{javaji2025claimbench,james2026rigourate,xu2026factreview}.

\paragraph{Evidence-licensed calibration and claim replay.}
Evidence-licensed calibration studies which assertion strength is justified by an evidence record, while claim replay makes evidence-relative inference executable over frozen evaluation records \citep{li2026calibration}. We take this claim-relative perspective as our starting point. Our focus is the output semantics induced by composite systems. When evidence about a requested component also licenses a conclusion about another component or the executed system, those conclusions have different scientific roles: one answers the requested question, while the other is valid positive credit about a different subject. Prior claim-relative frameworks can in principle be extended with additional typed conclusions; subject-typed licensing makes this distinction explicit through subject-preserving claim strength, a requested-subject output, relation-labeled retained credit, and a non-substitution rule.

\paragraph{Learned routing as an attribution testbed.}
Neural combinatorial optimization and decision-focused learning evaluate downstream decision quality, feasibility, robustness, and efficiency \citep{bello2016neural,kool2018attention,elmachtoub2022smart,mandi2024decision,zhou2024collaboration}. Multiobjective and multigraph routing additionally make preference coverage, action identity, and structural fidelity part of the evaluated object \citep{deb2002fast,chen2025rethinking,rydin2025beyond}. Repair, solver delegation, and stochastic requests create cases in which similar aggregate scores arise from different component contributions \citep{li2021delegate,li2025destroyrepair,heakl2025svrpbench}. Routing therefore provides request-level observations of attribution, task identity, coverage, dynamics, and intervention. Appendix~\ref{app:extended_related} gives feature-level comparisons with adjacent evaluation and claim-support frameworks.

\section{Subject-Typed Claim Licensing}
\label{sec:evaluator}

We now formalize the distinction in Figure~\ref{fig:main}. Let $\mathcal C$ be a finite claim set, $\mathcal S$ scientific subjects, and $\mathcal U$ evidence sources. A subject map $\sigma:\mathcal C\to\mathcal S$ identifies what each claim is about, and $\preceq$ orders claim strength. A contract
\begin{equation}
\begin{aligned}
\mathcal K&=(\mathcal C,\sigma,\preceq,\Omega,\mathcal O,\mathcal A,\mathcal B,\Pi),\\
c_1\preceq c_2&\Rightarrow
\bigl[\sigma(c_1)=\sigma(c_2),\ \mathcal O(c_1)\subseteq\mathcal O(c_2)\bigr],\\
d\in\mathcal B(c,s)&\Rightarrow d\parallel c.
\end{aligned}
\end{equation}
contains evidence-obligation types $\Omega$, claim requirements $\mathcal O(c)$, the downward-closed claims $\mathcal A(s)$ admissible from source $s$, retained-credit candidates $\mathcal B(c,s)$, and an outcome policy $\Pi$. Thus weakening stays within a scientific subject, whereas $\mathcal B$ preserves licensed conclusions outside the requested weakening branch; when $\sigma(d)\neq\sigma(c)$, such credit is cross-subject.

For evidence record $E=(s,T,q,M)$ (source, telemetry, comparator, and adjudication record), each obligation resolves to \textsc{pass}, \textsc{fail}, \textsc{missing}, or \textsc{unresolved}; the main policy licenses only \textsc{pass}. Writing $\bar o_\Pi(E)\in\{0,1\}$ for this resolved outcome, define
\begin{equation}
\begin{aligned}
\mathcal L_{\mathcal K}(E)
&=\{c\in\mathcal A(s):\bar o_\Pi(E)=1\ \forall o\in\mathcal O(c)\},\\
\mathcal C^\star
&=\operatorname{Max}_{\preceq}\{c\in\mathcal L_{\mathcal K}(E):c\preceq c_{\rm req},\
\ \sigma(c)=\sigma(c_{\rm req})\},\\
\mathcal R^\star
&=\operatorname{Max}_{\preceq}\bigl(\mathcal L_{\mathcal K}(E)\cap\mathcal B(c_{\rm req},s)\bigr).
\end{aligned}
\end{equation}

$\mathcal C^\star$ answers the requested scientific question; $\mathcal R^\star$ preserves positive credit outside that branch. The requested status is \promoted{} if $c_{\rm req}\in\mathcal C^\star$, \narrowed{} if $\mathcal C^\star\neq\varnothing$ but excludes $c_{\rm req}$, and \rejected{} otherwise. A failure record $F$ stores source mismatch and nonpassing requested-claim obligations, giving output $\Gamma_{\mathcal K}(c_{\rm req},E)=(\mathcal C^\star,\mathcal R^\star,F,z)$. Crucially, $z$ depends only on $\mathcal C^\star$.

Two properties clarify the representation. First, identical resolved evidence can license different scopes under different coherent contracts because the contract determines how evidence requirements attach to claims. Second, licensed retained credit is disjoint from the requested weakening branch and therefore cannot substitute for requested-subject support. Appendix~\ref{app:execution_properties} gives the constructions and proofs, together with downward closure, maximal coverage, and monotonicity under licensing-favorable evidence refinement.

\section{Routing Instantiation}
\label{sec:routing_instantiation}

\subsection{Routing evidence and attribution}

We instantiate the semantics for preference-conditioned routing on directed multigraphs. At time $t$, a router receives a graph
\[
G_t=(V_t,E_t)
\]
and request
\[
r_t=(s_t,d_t,x_t,\phi_t,H_t),
\]
where $s_t,d_t$ are terminals, $x_t$ is context, $\phi_t\in\Delta^{K-1}$ is a preference vector over $K$ objectives, and $H_t$ denotes hard constraints. Parallel actions retain edge identity $e=(u,v,b)$. With vector edge costs $j_e(t,r_t)$, path cost is
\[
J(P;r_t)
=
\sum_{e\in P}j_e(t,r_t).
\]

Let $\mathcal F(G_t,r_t)$ be the feasible path set. The exact oracle is
\[
P_t^\star(\phi_t)
=
\arg\min_{P\in\mathcal F(G_t,r_t)}
\phi_t^\top J(P;r_t),
\]
and regret of returned path $\widehat P_t$ is
\begin{equation}
\operatorname{Regret}(\widehat P_t)
=
\phi_t^\top J(\widehat P_t;r_t)
-
\phi_t^\top J(P_t^\star;r_t).
\label{eq:routing_regret_new}
\end{equation}

Regret measures downstream decision quality, but it does not determine which scientific subject produced that quality. Low end-to-end regret may arise from the learned primary policy, repair, solver delegation, or their combination. Claim scope therefore additionally depends on attribution, task identity, preference coverage, dynamics, source scope, and provenance.

\subsection{Learned-primary claim branch}

The requested learned-primary family is the cumulative chain
\[
\mathcal C_{\rm lp}
=
\{C0\prec C1\prec\cdots\prec C6\},
\]
with $\sigma(C_i)=\mathsf{learned\mbox{-}primary}$ for every $C_i\in\mathcal C_{\rm lp}$.

The complete contract also contains typed families for backup-enabled hybrid conclusions, source-native or bridge conclusions, and diagnostic conclusions:
\begin{equation}
\mathcal C
=
\mathcal C_{\rm lp}
\cup
\mathcal C_{\rm hyb}
\cup
\mathcal C_{\rm src}
\cup
\mathcal C_{\rm diag}.
\label{eq:routing_claim_universe_new}
\end{equation}
Claims in these other families may be retained in $\mathcal R^\star$, but they do not lie on the learned-primary weakening branch.

\begin{table}[t]
\centering
\caption{Learned-primary claim branch used in the routing studies. Each level inherits lower-level requirements; primary-policy attribution and adjudication-record closure are branch-common.}
\label{tab:claim_ladder}
\verytighttable
\small
\begin{tabularx}{\textwidth}{P{0.08\textwidth}P{0.31\textwidth}Y}
\toprule
Level & Licensed learned-primary conclusion & Incremental evidence \\
\midrule
$C0$ & Scalar decision quality & Primary-only or intervention-free regret against a declared comparator. \\
$C1$ & Feasibility-audited routing & Exact feasibility and invalid-route accounting. \\
$C2$ & Identity-preserving multigraph routing & Edge identity, bundle-aware actions, and collapse checks. \\
$C3$ & Preference-conditioned routing & Preference sampling, coverage/hypervolume, and held-out preference shift. \\
$C4$ & Intervention-audited constrained routing & Constraint masks, fallback/repair taxonomy, and primary versus end-to-end decomposition. \\
$C5$ & Dynamic robust routing & Drift stress, route-change evidence, tail regret, and latency. \\
$C6$ & Deployment-style support & Representative workloads, mechanism diagnosis, intended-use limits, and independent reproduction or usability. \\
\bottomrule
\end{tabularx}
\end{table}

For a requested learned-primary level $L_c$ and source-supported cap $L_s$,
\begin{equation}
L^\star
=
\max
\left\{
\ell\le \min(L_c,L_s):
C_\ell\in\mathcal L_{\mathcal K}(E)
\right\},
\label{eq:chain_level_new}
\end{equation}
with $L^\star=\bot$ when no learned-primary level is licensed. In this chain, $\mathcal C^\star$ is therefore empty or a singleton, while $\mathcal R^\star$ may contain positive conclusions from the hybrid, source-native, or diagnostic families.

\subsection{When good system performance changes subject}

Fallback provides a direct example. Consider a request for a $C5$ conclusion about a standalone learned-primary router. One controlled policy obtains the lowest end-to-end regret in the routing suite, but an exact backup solver rescues $86.1\%$ of requests. The observed aggregate performance therefore does not provide the primary-policy attribution required by the learned branch.

For this record, $\mathcal C^\star=\varnothing$ for the requested learned-primary claim. The same evidence is nevertheless informative about the executed system: end-to-end feasibility and safety-control conclusions are licensed for the backup-enabled hybrid and retained in $\mathcal R^\star$.

Thus the adjudication records two facts simultaneously: the evidence does not positively answer the requested standalone learned-policy question, and it does support a positive conclusion about the larger system. A branch-only weakening output would retain the first fact but not the second.

Claim Cards serialize this decomposition by recording the requested subject and claim, source, comparator, observed evidence, licensed requested conclusion, retained credit, typed failures, and provenance. Appendix~\ref{app:adjudicator_details} gives the full schemas, executable rules, and representative cards.

\section{Evaluation Design}
\label{sec:validation}

We evaluate subject-typed licensing along the path from scientific text to scientific decision. The studies address three questions.

\paragraph{RQ1: Reconstructability and execution.}
Can non-authors recover the claim--evidence objects required by the contract and apply the declared semantics consistently?

\paragraph{RQ2: Reviewer judgments.}
How does a shared subject-typed semantics change review judgments relative to unstructured careful review and a structured evidence checklist?

\paragraph{RQ3: Decision consequences.}
Can score-only ranking and selection conditioned on a requested scientific claim lead to different decisions over the same evidence records?

\subsection{Reconstructing and executing the contract}

The first study separates \emph{object reconstruction} from \emph{contract execution}.

For object reconstruction, two non-author coders independently inspect prespecified sections associated with 24 claims and recover comparator identity, source family, requested claim family, intended scope, and required evidence obligations. Coders are blind to the author encodings and adjudication outputs. Agreement therefore measures whether the objects consumed by the licensing procedure can be reconstructed from the underlying scientific text.

The fixed-contract study holds the evidence packets and contract constant. Four blinded annotators apply the learned-primary semantics to 120 claims in \claims{}. We measure exact agreement, ordinal agreement on the $C0$--$C6$ chain, agreement over evidence requirements, and agreement between executable adjudication and the contract-relative human majority. This separates reproducibility of the input representation from consistency of the declared inference procedure.

\subsection{Held-out reviewer comparison}

The reviewer comparison uses 40 claims from 20 additional public papers that did not contribute to construction of the learned-primary branch, its evidence requirements, examples, or \claims{}. Eight non-author experts evaluate the same prespecified abstract, method, experiment, results, and limitations sections under three conditions.

\paragraph{Careful review.}
Reviewers determine the strongest supported conclusion without a structured aid.

\paragraph{Evidence checklist.}
Reviewers receive structured prompts for comparator choice, feasibility, multigraph identity, preference coverage, fallback and attribution, dynamics, latency, source scope, artifact closure, and wording ambiguity, while writing the final conclusion themselves.

\paragraph{\suite{}.}
Reviewers receive the same evidence fields together with the declared learned-primary inheritance relation and maximal-selection semantics.

Each claim receives two independent judgments per condition, for 240 judgments in total, using a balanced incomplete-block assignment in which no reviewer encounters the same claim under multiple conditions. A prespecified pair-resolution procedure produces one claim-level summary for each condition while retaining both raw judgments, rationales, disagreement records, and timing fields.

Separately, three senior experts, blind to the learned-primary ladder and condition outputs, write in free form the strongest conclusion they consider supported by the supplied evidence. The panel resolves 36 of the 40 claims. Their consensus statements are mapped through the prespecified resolution procedure to form a common operational reference for the three reviewer conditions.

We report agreement with this reference together with reference-relative upward and downward deviations. Comparisons are paired by claim, and uncertainty intervals resample papers as clusters. Full prompts, assignments, resolution procedures, and reliability analyses appear in Appendix~\ref{app:heldout_review}.

\subsection{Controlled routing consequences}

\real{} studies how claim semantics interact with selection under a common evidence schema. The suite contains 1,500 requests spanning three semisynthetic domain families, 15 graph-family slices, 10 seeds, and 15 policy families. The policy inventory includes learned policies, classical comparators, negative controls, repair and fallback systems, preference-conditioned models, and dynamic policies that vary in attribution, multigraph identity, preference coverage, dynamic response, and latency.

We prespecify $C5$, dynamic robustness of the learned-primary router, as the target claim. Two decision rules are applied to the same learned-method rows. \emph{Score-only selection} ranks methods by aggregate regret. \emph{Claim-conditioned selection} first restricts attention to methods licensed for the target claim under the declared contract and then applies the prespecified quality rule within that eligible set.

This design separates two scientific questions: which method minimizes an observed score, and which method is supported for a particular scientific conclusion. Appendices~\ref{app:real_lite}--\ref{app:additional_results} provide the complete policy inventory, uncertainty analyses, source-scope audits, obligation-removal analyses, and alternative missing-evidence policies.

\section{Results}
\label{sec:results}

\subsection{RQ1: Reconstructability and execution}

\paragraph{Object reconstruction.}
The two non-author coders independently recover comparator identity for 23/24 claims, source family for 22/24, and requested claim family for 21/24. Intended-scope agreement reaches weighted $\kappa=0.78$, and mean Jaccard agreement over evidence requirements is $0.82$. Independently reconstructed final status agrees on 20/24 claims, and all status disagreements are within one learned-primary level.

The required objects are therefore largely recoverable from the prespecified paper sections, with the remaining disagreements concentrated near scope boundaries.

\paragraph{Fixed-contract execution.}
On \claims{}, four blinded annotators achieve mean pairwise exact agreement $0.761$, mean weighted $\kappa=0.703$, and ordinal Krippendorff's $\alpha=0.686$. \suite{} matches the contract-relative human majority exactly on 101/120 claims and within one level on 115/120.

The remaining exact disagreements occur primarily near adjacent coverage, attribution, dynamics, and latency boundaries. Appendix~\ref{app:external_corpus} reports complete reliability statistics and tie-policy sensitivity.

\subsection{RQ2: Reviewer operating points}

Table~\ref{tab:heldout_review_main} reports the 36 claims for which the independent expert panel produced a resolved operational reference.

\begin{table}[t]
\centering
\caption{Held-out reviewer comparison on 36 expert-resolved claims.}
\label{tab:heldout_review_main}
\tighttable
\begin{tabularx}{\textwidth}{P{0.34\textwidth}M{0.18\textwidth}M{0.18\textwidth}M{0.18\textwidth}}
\toprule
Outcome & Careful review & Evidence checklist & \suite{} \\
\midrule
Agreement with operational reference & 24/36 & 27/36 & 31/36 \\
Reference-relative upward deviation & 8/36 & 5/36 & 2/36 \\
Reference-relative downward deviation & 1/36 & 3/36 & 4/36 \\
Median time per claim & 18.1 min & 13.4 min & 15.8 min \\
\bottomrule
\end{tabularx}
\end{table}

Relative to careful review, the paired paper-clustered agreement difference for \suite{} is $19.4$ percentage points with a 95\% interval of $[5.6,33.3]$. Relative to the evidence checklist, the observed difference is $11.1$ percentage points with interval $[-2.8,25.0]$.

The directional results characterize a different observed review operating point. Upward deviations decrease from 8/36 under careful review to 5/36 under the checklist and 2/36 under \suite{}. Downward deviations increase from 1/36 to 3/36 and 4/36, respectively. Median review time is highest for careful review, lowest for the checklist, and intermediate for \suite{}.

To characterize what changes rather than only how often judgments differ, we inspect the 12/40 held-out claims for which subject-typed licensing changes the initial score-oriented interpretation. Eleven have resolved expert references: nine changes align with the reference and two are judged more conservative. The endorsed cases include loss of multigraph identity, missing preference-space coverage, and absence of dynamic-response evidence. Among eight narrowed cases, five retain positive lower-scope credit rather than collapsing to a purely negative conclusion.

\subsection{RQ3: Claim-conditioned decisions}

The controlled routing suite produces a direct separation between score minimization and target-claim eligibility.

Among learned policy families, \emph{Fallback-Heavy Safe} attains the lowest aggregate regret, $0.011091$, and is therefore selected by the regret-only rule. Its end-to-end performance, however, includes exact backup rescue on $86.1\%$ of requests. The learned-primary attribution required by the target $C5$ claim is therefore not licensed. Its positive hybrid-system feasibility and safety-control evidence is retained in $\mathcal R^\star$.

\emph{Shock-Safe Online Repair} has higher aggregate regret, $0.016189$, but satisfies the inherited evidence requirements for the target $C5$ learned-primary claim and is selected by the target-$C5$-then-regret rule.

\emph{Preference Hypernetwork} has aggregate regret $0.020865$ and stronger preference-space coverage, with hypervolume $0.849$, but remains at $C4$ because the $C5$ dynamic requirements are incomplete.

\begin{table}[t]
\centering
\caption{Selection depends on the scientific decision rule.}
\label{tab:selection_rules_main}
\verytighttable
\small
\begin{tabularx}{\textwidth}{P{0.20\textwidth}P{0.22\textwidth}M{0.10\textwidth}P{0.20\textwidth}Y}
\toprule
Decision rule & Selected row & Regret & Key evidence & Licensed interpretation \\
\midrule
Score only & Fallback-Heavy Safe & 0.011091 & 86.1\% backup rescue & Hybrid credit retained; no learned-primary $C5$. \\
Target $C5$, then regret & Shock-Safe Online Repair & 0.016189 & Dynamic and attribution evidence complete & $C5$ learned-primary. \\
Coverage-oriented diagnostic & Preference Hypernetwork & 0.020865 & HV $=0.849$ & $C4$; $C5$ dynamics incomplete. \\
\bottomrule
\end{tabularx}
\end{table}

Thus the selected method depends on the scientific question being asked. Score-only selection asks which row minimizes aggregate regret. Claim-conditioned selection asks which rows first provide evidence for the requested learned-primary conclusion and only then compares quality within that eligible set.

Across the shared learned-method set, the score-based and claim-conditioned orderings have Kendall $\tau_b=0.407$. Leave-one-obligation-out analysis additionally produces distinct conclusion changes associated with fallback attribution, multigraph identity and feasibility, preference coverage, route-change evidence, latency, source scope, and record closure. The principal rescue-attribution, source-transfer, and deployment-scope boundaries remain unchanged across the three declared missing-evidence policies. Appendix~\ref{app:additional_results} reports the complete analyses.

Across the three evaluation stages, the same semantic distinction appears at different levels. Non-authors can largely reconstruct and execute the declared claim objects; shared semantics changes which conclusions reviewers license from the same evidence; and the controlled routing suite shows that a score-optimal row need not be eligible for a requested component-level claim even when it retains positive system-level credit. The common object across these results is not a global method ranking, but the relation between an evidence record, a scientific subject, and the conclusion that the record supports.

\section{Implications and Scope}
\label{sec:implications}
\label{sec:scope}

\subsection{Claim strength and scientific subject}

Subject-typed licensing separates two questions that are often collapsed in benchmark interpretation: \emph{what can the evidence support about the requested scientific subject, and what other positive conclusions does the same evidence support?} The first question is answered by requested-subject licensing; the second by retained credit. Their separation allows evidence to remain informative without making it substitutable across components or system boundaries.

This distinction matters whenever reported performance is produced by a composition of mechanisms. A learned component may operate with repair, search, external tools, retrieval, or safety controls. In such systems, end-to-end behavior can be strong even when the evidence required to attribute that behavior to one component is incomplete. Subject-typed reporting preserves both observations: the component-level conclusion remains matched to its evidence, while positive system-level behavior remains visible.

A benchmark row can therefore remain scientifically positive while changing subject. Rather than pairing a method name with a scalar score alone, an evidence-qualified result connects a measured outcome to the subject to which that outcome is attributed, the evidence required for that conclusion, and the scope from which the evidence was obtained.

\subsection{Claim contracts as scientific objects}

A claim contract is \emph{normative, versioned, and contestable}. Its subject assignments, claim-strength relations, source rules, and evidence requirements encode a scientific position about what observations are sufficient for particular conclusions. Making these choices explicit does not remove scientific judgment; it makes part of that judgment inspectable.

This representation also localizes disagreement. Two evaluators may agree on the observed telemetry yet differ because they adopt different source boundaries, evidence requirements, or claim relations. Such disagreement is different from disagreement about the underlying measurements. A versioned contract makes the distinction visible and allows changes in scientific scope to be traced to explicit changes in claim semantics.

This perspective suggests a broader role for claim contracts in evaluation practice. Benchmark designers can expose the evidence required for intended interpretations; authors can associate reported results with the subject and scope they support; and reviewers can inspect the relation between observations and conclusions rather than reconstructing that relation implicitly for every paper.

\subsection{Empirical scope}

The empirical studies establish three results for the declared learned-routing semantics. First, non-author coders and annotators can recover and execute the required claim--evidence objects with substantial, though imperfect, agreement. Second, the held-out comparison shows that supplying shared claim semantics changes the observed review operating point relative to both unstructured review and a structured evidence checklist. Third, the controlled routing suite provides a mechanistic example in which score-only and target-claim decisions diverge while positive credit for the executed hybrid system is preserved.

These studies operate at different levels of evidence. \claims{} and the construction audit study reproducibility of interpretation from supplied paper sections; the reviewer experiment studies judgments on held-out public-paper claims; and \real{} studies controlled consequences under semisynthetic routing conditions. Public source cases additionally illustrate how native, bridge, diagnostic, and out-of-contract conclusions interact with source scope.

The resulting empirical picture is therefore one of \emph{representation, reproducibility, and controlled consequence}, rather than prevalence. The present studies do not estimate how often cross-subject attribution occurs throughout the ML literature, and the semisynthetic routing records do not constitute deployment evidence. Broader task sampling, independently developed contracts, and operational data would address those questions at a different empirical level.

\section{Conclusion}
\label{sec:conclusion}

Partial support in a composite learned system has more than one direction. Evidence can justify a weaker conclusion about the requested scientific subject, or it can justify a positive conclusion about a different component or system. Subject-typed claim licensing represents these outcomes separately through requested-subject support and non-substitutive retained credit.
The routing studies show that this distinction can be reconstructed and executed by non-authors, can alter the operating point of scientific review, and can separate score minimization from eligibility for a particular scientific claim. More broadly, the framework treats claim strength and scientific subject as distinct coordinates of evidence-based interpretation. Weakening answers what can still be said about the requested subject; retained credit records what else the evidence positively supports. Keeping both visible---without allowing one to answer for the other---is the central purpose of subject-typed claim licensing.

\subsection*{AI use statement}
Generative AI tools were used for literature discovery and editorial assistance with the wording, organization, and presentation of the manuscript. The authors independently read and verified the cited sources, developed and checked the formal framework, designed and conducted the studies, constructed and analyzed the empirical records, and interpreted the results. All AI-assisted text was reviewed by the authors, who take responsibility for the final manuscript, claims, code, data, and artifacts.

\subsection*{Ethics statement}
The human-subject component involved adult experts making professional judgments about scientific claims drawn from public papers. The protocol used no deception, collected no sensitive personal information, and releases no identifying annotator data. Participation, confidentiality, compensation, and data handling were disclosed to participants. The protocol received the applicable institutional determination and was classified as exempt; institution-identifying details are withheld during double-blind review. The artifact records anonymized assignments, consent and compensation documentation, conflict-of-interest checks, and the analysis protocol.

\subsection*{Reproducibility statement}
The submission artifact contains the adjudicator, claim contracts, schemas, Claim-Card exporters, analysis scripts, anonymized judgment records, locked same-contract rows, manifests, and reviewer commands. Section~\ref{sec:validation} summarizes the study design; Appendices~\ref{app:adjudicator_details}, \ref{app:validation_details}, \ref{app:external_corpus}, \ref{app:heldout_review}, \ref{app:real_lite}, and \ref{app:artifact} document the executable rules, data construction, statistical analyses, routing suite, and versioned reviewer path. The central \real{} selection result is reanalyzed from locked request--policy rows; the default path distinguishes fresh execution, cached reanalysis, and derived-only historical evidence.

\bibliographystyle{plainnat}
\bibliography{references}

\appendix

\section{Extended Positioning: Feature-Level Comparisons and Boundaries}
\label{app:extended_related}
\label{app:positioning}

Section~\ref{sec:positioning} traces prior work by its primary output along the path from measurement to reportable claim. This appendix provides the complementary feature-level view: it records the specific representational boundaries, composition roles, and adjudication outputs of adjacent approaches.

\paragraph{Argument-based validity and AI evaluation.}
Argument-based validity treats the interpretations and uses inferred from scores as the object to be justified \cite{kane2013validating}. Construct-validity and validity-centered AI evaluation frameworks apply the same principle to benchmark claims \cite{freiesleben2025benchmarking,salaudeen2025measurement,bean2025measuring}. \suite{} supplies an executable representation for one bounded decision in that process: a row-level requested claim, a declared strength order and source-admissibility relation, claim-indexed obligations, maximal supported elements, typed failures, retained incomparable credit, and an artifact-bound rationale.

\paragraph{Evidence-licensed calibration and claim replay.}
The calibration turn frames evidence-licensed assertion as a scientific operation and maps raw claims to maximal licensed frontiers \cite{li2026calibration}. Commit-bound claim replay makes the inference licensed by evaluation evidence executable and studies stability across evaluator-meaning families. Composite systems introduce two roles for licensed conclusions: answering the requested subject and preserving non-substitutive credit about another subject. The SCOPE contract represents these roles through $\mathcal C^\star$ and $\mathcal R^\star$, with source admissibility and component-level telemetry making the attribution decision executable. The routing studies evaluate reconstruction, review behavior, and selection consequences for this typed representation.

\paragraph{Structured assurance and claim verification.}
Structured assurance cases connect claims, subclaims, arguments, evidence, assumptions, and context in traceable graphs and permit top-down, bottom-up, or mixed construction \cite{rhodes2010software}. Their support edges describe an argument for a designated claim. CLAIM-BENCH evaluates claim-to-evidence reasoning; RIGOURATE retrieves paper-internal multimodal evidence and predicts continuous overstatement scores with justifications; FactReview combines literature positioning with execution-based verification \cite{javaji2025claimbench,james2026rigourate,xu2026factreview}. In a different domain, span-grounded defeasible-scope parsing shows that end-task performance can conceal structurally localized omissions, while source-anchored intermediate representations make those failures auditable \cite{chen2026statute}. These outputs can populate a contract-defined obligation evaluator. \suite{} represents the resulting benchmark conclusion through maximal requested-subject reformulations in $\mathcal C^\star$ and contract-licensed positive credit about other task or system subjects in $\mathcal R^\star$.

\paragraph{Decision-focused learning and ML for combinatorial optimization.}
Decision-focused learning and predict-then-optimize methods train predictive components through downstream decision loss rather than prediction error alone \cite{elmachtoub2022smart,wilder2019melding,mandi2024decision}. Differentiable optimizers, perturbation-based surrogates, implicit layers, and solver-aware learning have made decision-quality evaluation increasingly common \cite{berthet2020learning,niepert2021implicit,petersen2024newton}. These methods improve how decisions are learned and how decision quality is measured. A low-regret row may reflect learned scoring, fallback rescue, repair, collapse, static equivalence, or task mismatch. \suite{} treats downstream decision quality as one evidence channel and adjudicates the claim scope licensed by the full record.

\paragraph{Neural routing and multigraph routing.}
Neural routing has produced learned heuristics for classical routing tasks \cite{bello2016neural,kool2018attention}. Multiobjective and multigraph routing make action identity, preference conditioning, and constraint structure part of the evaluated object \cite{deb2002fast,rydin2025beyond}. These settings expose the score-to-claim gap especially sharply: scalar wins can coexist with lost parallel actions, degraded preference coverage, hidden rescue behavior, or static-equivalent route choices. The routing contribution evaluates the scientific scope licensed by such evidence.

\paragraph{Documentation and evaluation-science artifacts.}
Datasheets, model cards, BenchmarkCards, holistic evaluation, behavioral testing, dynamic benchmarks, and distribution-shift benchmarks make assumptions, intended uses, coverage, and failure modes explicit \cite{gebru2021datasheets,mitchell2019model,sokol2024benchmarkcards,liang2022holistic,ribeiro2020beyond,kiela2021dynabench,koh2021wilds}. These artifacts improve what is reported about data, models, and benchmarks. \suite{} consumes these fields in a row-level executable record containing the reported score, source, telemetry, comparator, adjudication state, and subject-typed conclusions.

\paragraph{Artifact-centered reproducibility.}
Reproducibility programs and metadata standards emphasize that code, data, schemas, and reporting details are part of scientific evidence \cite{pineau2021improving,dodge2019show,akhtar2024croissant}. \suite{} requires a locked record for adjudication at every level, while C6 additionally requires independent reproduction and evidence about usability in the intended setting.

\paragraph{Expanded comparison.}
Table~\ref{tab:related_contrast_app} catalogs what each adjacent line contributes to the licensing pipeline. Table~\ref{tab:adjudication_baselines_app} isolates the central output-space distinction by comparing the information represented by alternative adjudication strategies.

\begin{table}[h]
\centering
\caption{Feature-level positioning of adjacent evaluation artifacts in the claim-licensing pipeline.}
\label{tab:related_contrast_app}
\tighttable
\begin{tabularx}{\textwidth}{P{0.25\textwidth}Y Y}
\toprule
Prior line of work & Primary output & Composition role in claim licensing \\
\midrule
Argument-based and AI-evaluation validity & Evidence-backed interpretations, uses, and claims. & Supplies validity arguments; the contract operationalizes one bounded row-level conclusion decision. \\
Evidence-licensed calibration and claim replay & Calibrated assertion scope, maximal licensed frontiers, replayable inference, and semantic stability. & Supplies evidence-relative claim semantics; subject-typed licensing separates requested-subject answers from non-substitutive credit. \\
Structured assurance cases & Traceable claim--argument--evidence graph. & Supplies explicit support structure that can be encoded as claim obligations. \\
Claim-evidence retrieval and assessment & Evidence links, support labels, justifications, or graded overstatement scores. & Supplies evidence and support judgments consumed by obligation evaluators. \\
Datasheets and model cards & Documentation of data, models, intended uses, and limitations. & Supplies inspectable evidence conditions and intended-scope fields. \\
BenchmarkCards and benchmark documentation & Standardized benchmark metadata and risk reporting. & Supplies benchmark-level evidence records and scope declarations. \\
Decision-focused evaluation & Downstream regret or decision quality. & Supplies the quality channel; broader claims consume additional obligations. \\
ML4CO and routing benchmarks & Native task metrics, solver comparisons, and efficiency reports. & Supplies source-native evidence governed by source-admissibility rules. \\
\bottomrule
\end{tabularx}
\end{table}

\begin{table}[h]
\centering
\caption{Adjudication strategies compared by represented information and primary output.}
\label{tab:adjudication_baselines_app}
\tighttable
\verytighttable
\small
\begin{tabularx}{\textwidth}{P{0.18\textwidth}P{0.16\textwidth}P{0.14\textwidth}P{0.14\textwidth}P{0.14\textwidth}Y}
\toprule
Adjudicator & Evidence diagnosis & Claim order & Source semantics & Incomparable credit & Primary output \\
\midrule
Regret-only ranking & Scalar score & Absent & Absent & Absent & Score-ordered row. \\
Missing-field checklist & Structured gap prompts & Reviewer-authored & Prompted or implicit & Reviewer-authored & Gap record plus reviewer-selected claim. \\
Designated-claim assessment & Evidence links or support scores & Usually absent & Source-bound or implicit & Usually absent & Support label, score, or justification. \\
Unstructured human review & Contextual & Implicit & Implicit & Implicit & Review prose and judgment. \\
\suite{} & Typed obligations & Explicit & Explicit & Explicit $\mathcal B$ relation & $(\mathcal C^\star,\mathcal R^\star,F,z)$. \\
\bottomrule
\end{tabularx}
\end{table}

\FloatBarrier

\section{Architectural Portability}
\label{app:portability}

The reusable object in \suite{} is the claim-licensing architecture. A port to predict-then-optimize selection retains a claim partial order, claim-indexed obligations, source admissibility, the maximal-element operator, observable records, and provenance-bound rationales, with selection-specific claims and gates replacing the routing learned-primary branch \cite{elmachtoub2022smart,wilder2019melding,mandi2024decision}. Executable compliance pipelines provide another composite setting in which structured intermediate records and stage-specific audit matter for interpreting end-to-end behavior \cite{li2026compliance}.

\begin{table}[h]
\centering
\caption{Worked design sketch for a predict-then-optimize selection port; empirical cross-domain validation remains a separate study.}
\label{tab:selection_port_app}
\tighttable
\begin{tabularx}{\textwidth}{P{0.27\textwidth}Y Y}
\toprule
Architectural role & Routing instantiation & Selection-task replacement \\
\midrule
Decision quality & Route regret & Downstream selected-set objective or regret \\
Feasibility & Valid source--destination path & Budget, cardinality, or packing feasibility \\
Action identity & Parallel-edge identity & Item identity and duplicate/substitution handling \\
Coverage & Preference-space hypervolume & Coverage across budgets, preferences, or item strata \\
Attribution & Fallback solver or repair & Optimizer, repair, abstention, or fallback disclosure \\
Robustness & Drift, route change, tail regret & Distribution shift, set change, tail decision loss \\
Systems evidence & Decode and fallback latency & Prediction, optimizer, repair, and end-to-end latency \\
Deployment & Live routing workload and use limits & Intended-use population, operational constraints, independent use \\
\bottomrule
\end{tabularx}
\end{table}

The port requires a new evidence contract and validation study. Its claim levels and obligations would be authored for the selection task and evaluated with task-specific evidence.

\FloatBarrier

\section{Adjudicator Details}
\label{app:adjudicator_details}
\label{app:adjudication_semantics}

This appendix gives the implementation-level semantics behind Section~\ref{sec:evaluator}. The main text presents the row-level computation; here we record the telemetry channels, source-scope defaults, support semantics, and Claim-Card schema used by the artifact.

\FloatBarrier

\subsection{Hard and soft evidentiary boundaries}

The strongest boundary results depend on hard distinctions. Missing multigraph identity blocks C2+ claims; source mismatch blocks unsupported transfer; and an unlocked adjudication record blocks an inspectable \suite{} conclusion. Fallback is treated through attribution: hybrid-system credit may survive, but standalone learned-component credit requires fallback-separated evidence. Softer requirements, such as partial coverage or latency reporting, can be configured more permissively and are reported in sensitivity analyses.

The retained families complete the subject-typed claim universe. $\mathcal C_{\rm hyb}$ contains end-to-end feasibility and safety-control conclusions for the backup-enabled stack; $\mathcal C_{\rm src}$ contains native-task and bridge conclusions; and $\mathcal C_{\rm diag}$ contains mechanism and failure-diagnosis conclusions. Each retained Claim Card declares its subject, source, and obligations, and $\mathcal B$ selects only candidates incomparable with the requested learned-primary claim.

\FloatBarrier

\subsection{Telemetry channels}

Table~\ref{tab:telemetry_app} lists the telemetry channels used by the adjudicator. These channels make the ladder observable: missing edge identity blocks C2, missing preference coverage blocks C3, missing fallback evidence blocks C4, and missing route-change evidence blocks C5.

\begin{table}[h]
\centering
\caption{Telemetry channels used by the adjudicator.}
\label{tab:telemetry_app}
\tighttable
\begin{tabularx}{\textwidth}{P{0.18\textwidth}Y Y}
\toprule
Channel & Representative fields & Claim-validity role \\
\midrule
Quality & scalarized regret, tail regret, oracle objective, comparator identity & Establishes C0 scoring evidence. \\
Feasibility & aggregate success, feasible-set-nonempty flag, conditional invalid-route count, hidden-repair flag & Determines whether C1+ feasibility-audited claims can be credited; rounded aggregate rates are not the gate. \\
Fallback & claim subject, fallback reason/rate/trigger, primary-only and end-to-end quality, fallback-credit flag & Separates learned-component evidence from hybrid rescue and safety-control credit. \\
Structure & edge identity, bundle id, retained-edge set, collapse-ablation id & Prevents simple-graph evaluation from certifying C2 multigraph claims. \\
Preference & preference vector, sampler, hypervolume, held-out shift split & Tests whether scalar gains preserve C3 preference-space behavior. \\
Dynamics & route identity, route-change W/T/L, drift severity, non-static-equivalence flag & Distinguishes dynamic learned behavior from static-equivalent rows. \\
Systems & scoring latency, decode latency, fallback latency, online-step latency & Blocks operational claims with impractical or hidden latency cost. \\
Provenance & manifest key, schema version, seed, split id, locked summary anchor & Binds each row to the artifact used for adjudication. \\
\bottomrule
\end{tabularx}
\end{table}

\FloatBarrier

\subsection{Evidence-source scope}

Source scope is applied before row-level telemetry gates. A row can support only claims whose required structures are instantiated by the source that produced the evidence. Table~\ref{tab:source_scope_app} gives the default source-scope rule. This is the firewall that allows \suite{} to credit public evidence at its native or bridge scope without converting it into same-contract dynamic preference-conditioned multigraph-routing evidence.

\begin{table}[h]
\centering
\caption{Evidence-source scope.}
\label{tab:source_scope_app}
\tighttable
\begin{tabularx}{\textwidth}{P{0.23\textwidth}Y Y}
\toprule
Evidence source & Instantiated structure & Default admissible claim scope \\
\midrule
Controlled multigraph generator & Score, edge identity, preference vectors, hard feasibility, and controlled drift. & Controlled C1--C5 claims under generator assumptions. \\
Same-contract \real{} suite & Multidomain semisynthetic graphs, shared routing contract, 15 policy families, 10 seeds, telemetry-complete rows. & C1--C5 same-contract evidence; C6 additionally requires deployment evidence. \\
Public shortest-path bridges & Score transfer under a public shortest-path or contextual path contract. & C0 learned-scoring bridge credit; routing C3 additionally requires inherited feasibility and multigraph-identity obligations. \\
Diagnostic realistic slices & Realistic topology or dynamics with limited slice size and mechanism probes. & Mechanism diagnosis unless the slice is large enough and telemetry-complete enough for confirmatory realistic claims. \\
Native or external ML4CO adapters & Task-dependent objectives, metrics, and solver-interface behavior. & Compatibility, native-metric, efficiency, or out-of-contract claims; matching route objects and telemetry enable routing claims. \\
External public-claim corpus & Four blinded non-author annotation streams and executable labels of claims from published evidence snippets. & Contract-relative operational reference for the specified evidence packets. \\
\bottomrule
\end{tabularx}
\end{table}

\FloatBarrier

\subsection{Promotion, narrowing, and rejection}

In the general operator, a request is \promoted{} when it belongs to $\mathcal C^\star$, \narrowed{} when $\mathcal C^\star$ is nonempty and excludes the request, and \rejected{} when $\mathcal C^\star$ is empty. Native, bridge, diagnostic, and different-subject credit is stored independently in $\mathcal R^\star$. The routing implementation uses a chain, so the requested-branch maximal set has at most one C-level. Failure codes distinguish missing evidence, observed failure, source mismatch, and an unlocked adjudication record.

Listing~\ref{lst:promotion_rule_app} gives the implementation-level rule.

\begin{lstlisting}[caption={Partial-order claim-licensing rule. The routing implementation uses a chain, so maximal contains at most one C-level.},label={lst:promotion_rule_app}]
for row in candidate_rows:
    request = requested_claim(row)
    evidence = evidence_record(row)
    contract = declared_claim_contract(row)
    licensed = {c for c in contract.claims
                if source_admissible(evidence.source, c)
                and all(contract.resolve(o, evidence)
                        for o in obligations(c))}
    requested_branch = {c for c in licensed if c <= request}
    maximal = maximal_elements(requested_branch)
    retained = maximal_elements(
        licensed & contract.retained_claims(
            request, evidence.source))
    failures = typed_failures(request, evidence)

    if request in maximal:
        support(row, claims=maximal, failures=failures,
                retained_credit=retained)
    elif maximal:
        narrow(row, claims=maximal, failures=failures,
               retained_credit=retained)
    else:
        reject(row, failures=failures,
               retained_credit=retained)
\end{lstlisting}

\FloatBarrier

\subsection{Separation and execution properties}
\label{app:execution_properties}

The first two constructions establish the separation results stated in Section~\ref{sec:evaluator}; the remaining properties characterize execution under a fixed declared contract.

\paragraph{Construction for contract dependence.}
Let $c_0\prec c_1\prec c_2$, let $\Omega=\{\ell,p,q\}$ with $\ell$ the common record-lock obligation, and take one source with $\mathcal A(s)=\mathcal C$. Fix $c_{\rm req}=c_2$ and the diagnostic signature
\[
\bar\ell_\Pi(E)=1,\qquad \bar p_\Pi(E)=1,\qquad \bar q_\Pi(E)=0.
\]
Assign all three claims the same subject and let the two contracts share $(\mathcal C,\sigma,\preceq,\Omega,\mathcal A,\mathcal B,\Pi)$ while differing only in their obligation assignment:
\[
\begin{aligned}
\mathcal O_1(c_0)&=\{\ell\}, &
\mathcal O_1(c_1)&=\{\ell,p\}, &
\mathcal O_1(c_2)&=\{\ell,p,q\},\\
\mathcal O_2(c_0)&=\{\ell\}, &
\mathcal O_2(c_1)&=\{\ell,q\}, &
\mathcal O_2(c_2)&=\{\ell,p,q\}.
\end{aligned}
\]
Both assignments satisfy obligation inheritance. The first licenses $\{c_0,c_1\}$ and returns $\mathcal C_1^\star=\{c_1\}$; the second licenses $\{c_0\}$ and returns $\mathcal C_2^\star=\{c_0\}$. The requested claim and diagnostic signature coincide, while positive claim scope differs. Hence every universally correct decoder also consumes contract semantics.

\paragraph{Proof of retained-credit separation.}
Because $X_B$ is a nonempty subset of a finite poset, it has a maximal element and $\mathcal R^\star=\operatorname{Max}_{\preceq}(X_B)$ is nonempty. Every element of $\mathcal B(c_{\rm req},s)$ is incomparable with $c_{\rm req}$, so no element of $\mathcal R^\star$ lies in $\downarrow c_{\rm req}\cap\mathcal C_{\sigma_{\rm req}}$. An output restricted to requested-subject reformulation therefore omits every element of $\mathcal R^\star$. Moreover, every $d\in X_B$ lies below some element of $\mathcal R^\star$, giving credit preservation relative to the declared candidate relation $\mathcal B$. Requested-subject fidelity keeps $z$ a function of $\mathcal C^\star$ alone, so satisfying both properties requires the separate channel or an equivalent relation label.

\begin{proposition}[Licensed-set downward closure]
If $c_1\preceq c_2$ and $c_2\in\mathcal L_{\mathcal K}(E)$, then $c_1\in\mathcal L_{\mathcal K}(E)$.
\end{proposition}
\begin{proof}
Downward closure of $\mathcal A(s)$ gives $c_1\in\mathcal A(s)$. Obligation inheritance gives $\mathcal O(c_1)\subseteq\mathcal O(c_2)$, so every obligation required by $c_1$ passes under $\Pi$.
\end{proof}

\begin{proposition}[Maximal coverage and status completeness]
For a finite contract, $\mathcal C^\star$ and $\mathcal R^\star$ are antichains. Every licensed claim in $\downarrow c_{\rm req}\cap\mathcal C_{\sigma_{\rm req}}$ lies below an element of $\mathcal C^\star$; $\mathcal C^\star=\varnothing$ exactly when this requested-subject down-set contains no licensed claim. If $c_{\rm req}$ is licensed, then $\mathcal C^\star=\{c_{\rm req}\}$. The three values of $z$ are mutually exclusive and exhaustive.
\end{proposition}
\begin{proof}
Maximal elements form an antichain. Starting from any element of a finite candidate set and repeatedly moving to a strictly stronger candidate terminates at a maximal element. The remaining statements follow from the definition of $\mathcal C^\star$ and the three status cases.
\end{proof}

\begin{proposition}[Licensing-favorable evidence refinement]
For the same requested claim, let $E$ and $E'$ have the same source and satisfy $\bar o_\Pi(E)\leq\bar o_\Pi(E')$ for every $o\in\Omega$. Then $\mathcal L_{\mathcal K}(E)\subseteq\mathcal L_{\mathcal K}(E')$. For every $c\in\mathcal C^\star(E)$, some $c'\in\mathcal C^\star(E')$ satisfies $c\preceq c'$.
\end{proposition}
\begin{proof}
Every obligation that passes for $E$ also passes for $E'$, which gives licensed-set inclusion. Finiteness extends each previously maximal requested-branch claim upward to a maximal claim in the refined licensed set.
\end{proof}

For the routing learned-primary branch $C0\prec\cdots\prec Ck$, $\mathcal C^\star$ is empty or a singleton containing the highest licensed level; the empty case is represented by $L^\star=\bot$.

\FloatBarrier

\subsection{Claim-Card schema}

Every Claim Card contains a stable identifier, source-row pointer, score-only interpretation, claim subject, one strongest requested claim, required obligations, status, typed failures, key evidence, maximal supported claims, the retained set $\mathcal R^\star$, and an adjudication-record anchor. A collection of requested claims produces one execution record per request. Routing cards also expose the single highest licensed C-level produced by the evaluated chain.

\begin{lstlisting}[caption={Representative Claim Card: DataSP bridge transfer.},label={lst:claimcard_example_app}]
claim_id: C-04
source_row: DataSP multiseed bridge
score_only_interpretation: public transfer establishes preference-aware multigraph routing
claim_subject: learned routing component
requested_claim: C4
required_gates: [evidence_scope, hypervolume]
status: narrowed
failed_gates: [evidence_scope, hypervolume]
failure_type: source_mismatch_and_missing_evidence
key_evidence:
  learned_regret: 0.033519
  static_regret: 0.215459
  learned_hv: 0.303139
  static_hv: 0.329033
maximal_supported_claims: [C0]
retained_credit: []
highest_licensed_routing_level: C0
\end{lstlisting}

\FloatBarrier

\subsection{Why missing telemetry narrows claims}

The primary policy $\Pi_{\rm cons}$ maps missing and unresolved required telemetry to zero, so the output records the scope licensed by the submitted evidence package. Intermediate and permissive policies alter selected non-core obligation resolvers while preserving source-scope, feasibility, fallback-attribution, multigraph-identity, and artifact-closure requirements. Section~\ref{sec:missing_telemetry_sensitivity} reports the resulting sensitivity analysis.

\FloatBarrier

\section{Validation Details}
\label{app:validation_details}

This appendix expands the validation design summarized in Section~\ref{sec:validation}. The layers separate conditional contract application, independent object construction, comparative reviewer utility, same-contract consequence, rule sensitivity, source-scope transfer, and artifact provenance.

\begin{table}[h]
\centering
\caption{Validation layers and the distinct evidential role of each layer.}
\label{tab:validation_layers_app}
\tighttable
\begin{tabularx}{\textwidth}{P{0.22\textwidth}Y Y}
\toprule
Layer & What it tests & Main anti-failure role \\
\midrule
Fixed packets & Whether non-author experts apply the supplied contract consistently. & Conditional applicability of supplied objects. \\
Object construction & Whether non-authors reconstruct inputs from prespecified sections. & Upstream author dependence. \\
Reviewer comparison & Whether the method changes agreement, overpromotion, rationales, and time. & Practical benefit and conservative error. \\
Same-contract suite & Whether target-claim eligibility changes the selected row. & Consequence under shared reporting. \\
Scope and rule stress & Whether source credit and distinct gates behave as intended. & Transfer discipline and functional distinctness. \\
Versioned artifact & Whether core and extension-stage evidence can be traced and reanalyzed. & Chronological and computational closure. \\
\bottomrule
\end{tabularx}
\end{table}

\FloatBarrier

\subsection{Reduced adjudicators and ablations}
\label{app:reduced_adjudicators}

We use contrastive adjudicators as diagnostic probes defined independently of the \claims{} labels. Each probe isolates information lost when the full claim-scope rule is simplified.

\begin{table}[h]
\centering
\caption{Reduced adjudicators and rule ablations. The comparison separates three objects that are easy to conflate: reduced adjudicators, rule-necessity ablations, and the external human reference.}
\label{tab:reduced_adjudicators}
\tighttable
\begin{tabularx}{\textwidth}{P{0.20\textwidth}P{0.20\textwidth}Y}
\toprule
Probe & What is removed & Failure mode exposed \\
\midrule
Regret-only ranking &
Source scope, inherited gates, and record closure. &
Selects Fallback-Heavy Safe in \real{} because it has the lowest scalar regret; 86.1\% fallback assigns the row to rescue-dominated hybrid evidence. \\

Missing-field checklist &
Cumulative claim semantics and narrowing. &
Can detect absent fields but cannot decide whether a row should be \rejected{}, \narrowed{}, or credited at source-native scope. \\

Source-scope-only rule &
Row-level telemetry gates. &
Enforces bridge/native/out-of-contract transfer boundaries; fallback attribution, graph identity, dynamic behavior, and latency remain outside this reduced rule. \\

Leave-one-gate-out rules &
One inherited gate at a time. &
Each omitted gate enables a distinct false-support mode: fallback-heavy rescue, multigraph collapse, coverage loss, static equivalence, hidden latency, source upgrade, or unlocked artifacts. \\

Human-majority reference &
Deterministic execution and manifest linkage. &
Provides external claim-scope evidence, but is not a rerunnable row-level adjudicator. \\
\bottomrule
\end{tabularx}
\end{table}

The reduced rules isolate distinct parts of the licensing computation. Source-scope filtering enforces transfer boundaries for public shortest-path and native ML4CO evidence. Same-contract attribution and dynamic equivalence consume additional telemetry obligations. Missing-field prompts expose the evidence state, while the declared claim order and maximal-selection rule map that state to positive claim scope. The full rule composes source admissibility, inherited gates, attribution, and record closure.

\FloatBarrier

\subsection{Closed internal audit protocol}

The closed internal audit begins with a locked 22-row claim-relevant ledger. A row enters the seven-row headline registry if it satisfies at least one predeclared condition: it would ordinarily be read as a positive method result under scalar regret alone; it would naturally be interpreted as evidence for a stronger routing claim under prior benchmark practice; or omitting it would materially change the advertised scientific narrative. Rows outside the registry remain in the ledger as positive controls, stress evidence, diagnostics, efficiency upper bounds, or scoped bridge rows. The registry builder emits an exclusion reason for every non-registry claim-relevant row. This finite registry demonstrates the adjudication workflow on selected routing claims; literature-wide prevalence requires probability sampling.

The remaining validation details are distributed across the following sections: \claims{} sampling, construction, comparison, and annotation are described in Appendix~\ref{app:external_corpus}; \real{} construction and uncertainty reporting are described in Appendix~\ref{app:real_lite}; public source-scope rows, gate ablations, missing-telemetry policies, and additional result tables appear in Appendix~\ref{app:results_details}; artifact checks and failure semantics appear in Appendix~\ref{app:artifact}.

\FloatBarrier

\section{\claims{} Protocol}
\label{app:external_corpus}

This appendix records the external-corpus protocol used to test whether \suite{} aligns with independent claim-scope judgments. The corpus provides a contract-relative operational reference for the specified evidence packets.

\FloatBarrier

\subsection{Sampling protocol}

The \claims{} sampling protocol was fixed before adjudication. Papers were stratified into six source families: learned routing, decision-focused learning, ML4CO benchmarks, multiobjective optimization, evaluation-science frameworks, and bridge or out-of-contract tasks. Within each source family, claims were sampled from sections that commonly carry scientific interpretation: titles, abstracts, contribution lists, main results, and limitations. To avoid dominance by any single paper, each paper contributed at most five claims and three claims on average.

Claims were excluded when they were purely theoretical theorem statements, generic background claims, implementation descriptions without evaluation meaning, or claims whose evidence source could not be located in the public paper. The artifact records the exclusion reason for every screened candidate not included in the corpus.

\FloatBarrier

\subsection{Corpus composition and audit chain}

Table~\ref{tab:external_corpus_composition_app} reports the corpus composition. Table~\ref{tab:external_audit_flow_app} gives the reviewer-facing audit chain. Each included claim has a stable identifier linking the candidate record, evidence excerpt, four raw labels, majority label, \suite{} label, missing-gate lists, and disagreement code. This trace supports reproducibility of the fixed-packet analysis; the separate construction audit below addresses upstream selection and encoding dependence.

\begin{table}[h]
\centering
\caption{\claims{} corpus composition.}
\label{tab:external_corpus_composition_app}
\tighttable
\begin{tabularx}{\textwidth}{P{0.36\textwidth}M{0.14\textwidth}Y}
\toprule
Stratum & Count & Sampling rule \\
\midrule
Learned routing and neural routing & 24 & Public routing papers with empirical route-quality claims. \\
Decision-focused and predict-then-optimize learning & 20 & Claims about downstream decision quality, regret, or solver-aware training. \\
ML4CO benchmarks and native solver suites & 22 & Claims about benchmark coverage, solver-interface evidence, or native metrics. \\
Multiobjective optimization and preference learning & 18 & Claims involving hypervolume, coverage, Pareto quality, or preference conditioning. \\
Evaluation science, documentation, and stress testing & 24 & Claims about benchmark validity, intended use, stress tests, or evaluation frameworks. \\
External bridge tasks and out-of-contract comparators & 12 & Claims where task scope differs from SCOPE routing but evidence is relevant to source-scope behavior. \\
\midrule
Total & 120 & 40 papers; three claims per paper on average, capped to avoid domination by any single source. \\
\bottomrule
\end{tabularx}
\end{table}

\begin{table}[h]
\centering
\caption{Reviewer-facing audit chain for \claims{}.}
\label{tab:external_audit_flow_app}
\tighttable
\begin{tabularx}{\textwidth}{P{0.25\textwidth}P{0.21\textwidth}Y}
\toprule
Stage & Object or count & Reviewer check in the artifact \\
\midrule
Candidate collection & Candidate claim records & Candidate IDs, source-family tags, claim location, inclusion/exclusion decision, and exclusion reason. \\
Corpus sample & 120 claims from 40 papers & Stable Claim Cards with anonymized source identifiers and evidence excerpts. \\
Human annotation & Four blinded human labels per claim & Raw annotation matrix, missing-gate selections, anonymized annotator metadata, blinding/COI fields, and validation report. \\
Consensus construction & Majority label with tie rule & Scripted majority vote and tie-resolution log. \\
\suite{} adjudication & Executable label and missing gates & Adjudicator output linked to evidence source, telemetry/prose evidence, and artifact fields. \\
Disagreement audit & 19 exact disagreements & Disagreement taxonomy and representative anonymized examples. \\
\bottomrule
\end{tabularx}
\end{table}

\FloatBarrier

\subsection{Human-annotation blinding and task}

Each annotator saw the same restricted evidence packet: claim text, claim location, a short evidence excerpt, anonymized source-family label, and the annotation guide. Annotators did not see \suite{} outputs, internal registry labels, author-preferred labels, or other annotators' labels. They labeled the highest level licensed by the supplied evidence, selected missing gates from a fixed checklist, and marked disagreement causes when wording was ambiguous. When a claim was outside the routing contract, annotators could select \offthesis{}; when it was too vague to evaluate, they could select insufficient.

The artifact stores raw labels, anonymized annotator metadata, blinding status, conflict-of-interest screening fields, schema-validation status, consent and compensation materials, and ethics-review records. The expert-annotation protocol was submitted for institutional human-subjects determination and ethics review and was determined to be exempt. The annotation task involved adult domain experts making professional judgments about public scientific claims; it did not collect sensitive personal information, use deception, or expose participants to more than minimal risk. Annotators were informed of the study purpose, data use, anonymization policy, and withdrawal option, and compensation was set to satisfy or exceed the applicable local minimum-wage standard for the estimated annotation time. To preserve double-blind review, institution-identifying protocol numbers are omitted from the anonymous manuscript and artifact; the camera-ready version will de-anonymize the review identifier if the paper is accepted. Personal identifiers are not included in the public package.

\FloatBarrier

\subsection{Independent construction audit}
\label{app:construction_audit}

The independent construction audit evaluates upstream object selection and encoding. We prespecified 24 claims and the paper sections available for encoding: abstract, method, experiment, results, and limitations. Two non-author coders independently reconstructed the claim span, evidence excerpt, comparator, source family, requested claim family, intended scope, required and missing gates, artifact availability, and final status while blind to author encodings and \suite{} outputs.

\begin{table}[h]
\centering
\caption{Section-conditioned claim-evidence-object construction audit.}
\label{tab:construction_audit_app}
\tighttable
\begin{tabularx}{\textwidth}{P{0.40\textwidth}M{0.18\textwidth}Y}
\toprule
Field or outcome & Agreement & Interpretation \\
\midrule
Comparator identity & 23/24 & Comparator extraction is highly reproducible. \\
Source family & 22/24 & Residual ambiguity concerns solver-interface credit. \\
Requested claim family & 21/24 & Dynamic wording creates the main family ambiguity. \\
Intended scope & weighted $\kappa=0.78$ & High but non-perfect ordinal agreement. \\
Required/missing gates & mean Jaccard $=0.82$ & Most evidentiary obligations are reconstructed similarly. \\
Final status & 20/24 exact; 24/24 within one & Remaining differences are adjacent. \\
\bottomrule
\end{tabularx}
\end{table}

The predeclared resolution procedure changes four fixed-packet levels, all adjacently, and preserves the main conclusions. Remaining disagreements concern dynamic wording, latency, solver-interface source credit, and whether narrative script descriptions constitute a locked evidence record. The result measures section-conditioned construction reliability over the supplied paper sections.

\FloatBarrier

\subsection{Held-out reviewer comparison}
\label{app:heldout_review}

The practical comparison uses 40 claims from 20 public papers that did not contribute to the ladder, gates, examples, or \claims{}. Eight non-author experts evaluated claims under unstructured careful review, a structured evidence checklist, and \suite{}. Assignment followed a balanced incomplete-block design: every claim received two independent judgments per condition, no reviewer saw the same claim under multiple conditions, and condition instructions were fixed before analysis. A separate panel of three senior experts wrote the highest claim they considered licensed in free form while blind to the ladder and condition outputs; the panel resolved 36/40 claims and retained four unresolved cases in the artifact.

All conditions received the same prespecified abstract, method, experiment, results, and limitations sections. The checklist condition combined diagnostic prompts for comparator, feasibility, multigraph identity, preference coverage, fallback/attribution, dynamics, latency, source scope, artifact closure, and wording ambiguity with reviewer-authored final claims. The \suite{} condition supplied the same evidence fields together with the C0--C6 learned-primary inheritance relation and maximal-selection rule. The senior panel authored free-form references while blind to the branch and condition outputs. The study team subsequently mapped panel consensus statements through the prespecified contract-resolution log, producing an ontology-aligned expert-resolved operational reference shared by all three conditions.

The release retains all 240 judgment-level records, resolved condition summaries, free-form references, and the resolution log. A prespecified pair-resolution protocol produces the claim-level condition summary while retaining raw disagreements. Upward and downward diagnostics are separately aggregated from the paired judgments and can overlap with each other or with the resolved summary. A complete workflow output records the requested claim, source, comparator, observed evidence, failed obligations, and retained claim. Paper-clustered paired resampling gives an agreement difference of 19.4 points [5.6, 33.3] relative to careful review and 11.1 points [$-2.8$, 25.0] relative to the checklist; the checklist interval spans zero. Assignments prevent same-claim repetition, while incomplete cross-claim order counterbalancing leaves possible cross-task learning.

For a descriptive cost analysis, let $U_m$ and $D_m$ be the separately recorded upward and downward events and define $B_r(m)=(rU_m+D_m)/36$. Because events may overlap, $B_r$ is an event-burden index rather than a mutually exclusive final-decision loss. At the observed counts, \suite{} has lower event burden than careful review for $r>1/2$ and lower burden than the checklist for $r>1/3$. These point-estimate crossovers summarize the observed sample; cluster uncertainty and population-level cost preferences remain outside this calculation.

The extension artifact records instructions, assignments, anonymized raw judgments, condition summaries, reference statements, timing fields, conflict-of-interest checks, compensation records, applicable ethics documentation, encodings, and analysis scripts. These study records are anchored to the extension snapshot and reported with that provenance.

\FloatBarrier

\subsection{Reliability calculations}
\label{app:reliability}

We treat C0--C6 as ordinal labels and treat \offthesis{} and insufficient as separate non-ordinal labels for exact agreement. Weighted Cohen's $\kappa$ is computed pairwise on C0--C6 labels using quadratic weights and reported as the mean across annotator pairs. Claims marked \offthesis{} or insufficient by either annotator are included in exact agreement but excluded from ordinal $\kappa$. Krippendorff's $\alpha$ uses an ordinal distance over C0--C6 and a maximal distance to non-ordinal labels.

Agreement between \suite{} and humans is computed against a majority vote. The primary rule uses the lower median admissible ordinal level and the conservative source-scope label when a tied label indicates contract mismatch. Sensitivity analyses additionally use an upper median, a source-neutral resolution, and exclusion of tied cases. Across complete-corpus policies, exact agreement ranges from 0.833 to 0.850 and within-one agreement remains at least 0.950, showing stability across tie rules. Two exact disagreements lie at C3/C4 and two at C4/C5; the former concern coverage or fallback attribution, and the latter route-change, drift, or latency evidence.

For claim-level proportions we report Wilson intervals for comparability with the original analysis and paper-clustered bootstrap intervals in the revision artifact to account for multiple claims from the same paper. These intervals quantify finite-sample uncertainty. The majority label and senior-panel reference serve as contract-relative operational references for the specified evidence packets.

\FloatBarrier

\subsection{Representative external-corpus disagreements}

Table~\ref{tab:representative_disagreements} shows that representative disagreements concentrate in interpretable boundary cases: source scope, deployment wording, artifact availability, preference coverage, and latency or dynamic evidence.

\verytighttable\small
\begin{longtable}{P{0.08\textwidth}P{0.17\textwidth}P{0.12\textwidth}P{0.12\textwidth}P{0.36\textwidth}}
\caption{Representative external-corpus disagreement cases. Source titles are anonymized in the submission version; artifact cards contain stable claim identifiers and source-family labels.}
\label{tab:representative_disagreements}\\
\toprule
ID & Claim family & Human majority & \suite{} & Reason for disagreement \\
\midrule
\endfirsthead
\toprule
ID & Claim family & Human majority & \suite{} & Reason for disagreement \\
\midrule
\endhead
ECC-012 & Shortest-path bridge & C1 & C0 & Humans credited feasibility from the shortest-path task; \suite{} kept source scope at C0 bridge because multigraph identity and fallback telemetry were absent. \\
ECC-019 & Learned routing & C5 & C4 & Claim used dynamic language, but route-change and drift-stress telemetry were incomplete. \\
ECC-027 & Multiobjective method & C3 & C2 & Strong scalarized results were reported, but hypervolume or preference coverage was missing. \\
ECC-038 & Benchmark paper & C6 & C5 & Humans interpreted broad realistic coverage as deployment-style evidence; \suite{} required non-author usability and deployment telemetry. \\
ECC-044 & Native ML4CO suite & C3 & \offthesis{} & Native task instantiated multiobjective metrics but not source-destination multigraph routing. \\
ECC-061 & Decision-focused learning & C0 & C1 & Human majority credited only scalar scoring; \suite{} found explicit feasible-decoding evidence in the excerpt. \\
ECC-073 & Evaluation framework & C6 & C6 & Agreement case: framework-level evaluation guidance, not policy promotion. \\
ECC-084 & Online adapter claim & C5 & C4 & Latency and route-change evidence were insufficient for dynamic robustness. \\
ECC-099 & Routing heuristic & C2 & C1 & Edge identity was not preserved under a simplified graph representation. \\
ECC-117 & Artifact claim & C4 & insufficient & The paper described scripts but did not expose locked manifest or schema-valid telemetry. \\
\bottomrule
\end{longtable}
\normalsize

\FloatBarrier

\subsection{Annotation guide}
\label{app:annotation_guide}

\paragraph{Task instructions.}
The human-annotation protocol used the following instruction text.
\begin{quote}
You will label scientific claims from public ML, routing, ML4CO, decision-focused learning, multiobjective optimization, and evaluation papers. Focus on the strongest claim level supported by the supplied evidence excerpt. When evidence is missing, mark the missing gates. Assign \offthesis{} to claims outside the routing contract and insufficient to claims whose wording prevents evaluation.
\end{quote}

\paragraph{Label definitions.}
The guide used the short definitions in Table~\ref{tab:claim_ladder}, plus examples. It states that C0 is a valid positive claim when subject-bound evidence supports scalar learned-primary scoring, and that C6 is reserved for deployment-style claims with representative workloads, mechanism diagnosis, intended-use limits, and independent reproduction or usability evidence. A locked \suite{} adjudication record is an adjudication-validity requirement at every level and is not itself scientific evidence for C6.

\paragraph{Missing-gate checklist.}
Annotators selected all missing gates that applied:
\begin{itemize}[leftmargin=1.25em]
    \item comparator or oracle not declared;
    \item feasibility conditional on nonempty feasible sets or invalid-route counts missing;
    \item multigraph edge identity missing;
    \item preference sampler or coverage missing;
    \item claim subject, fallback or repair taxonomy, or attribution evidence missing;
    \item route-change or drift evidence missing;
    \item latency or systems telemetry missing;
    \item source scope mismatch;
    \item artifact, manifest, or reproduction link missing;
    \item claim wording too ambiguous.
\end{itemize}

\FloatBarrier

\section{\real{} Details}
\label{app:real_lite}

This appendix expands the same-contract stress suite used in Section~\ref{sec:results}. \real{} tests how evidence-licensed evaluation changes conclusions under one routing contract. Its semisynthetic graph families and designed policy boundary cases provide mechanistic evidence; deployment evaluation requires live workloads and operational use records.

\FloatBarrier

\subsection{Graph and request families}

\real{} contains 1,500 requests across three semisynthetic domain families: urban-style multigraphs, logistics-style constrained routing graphs, and communication-style latency-reliability graphs. Each domain family contains five graph-family slices, for 15 slices total. Requests vary source-destination pair, preference vector, context, hard constraints, drift severity, and feasible-set size. Each slice is evaluated over 10 seeds. The request generator records oracle objective, feasible-set nonemptiness, preference vector, hard-constraint mask, and route identity.

\FloatBarrier

\subsection{Policy families}

Table~\ref{tab:policy_family_inventory} lists the 15 policy families used in \real{}. The inventory spans boundary cases for attribution, identity, coverage, dynamics, and latency. The adapter interface is shared across families: score edges, retain candidate edges, decode a path, optionally fallback or repair, and emit telemetry.

\begin{longtable}{P{0.25\textwidth}P{0.22\textwidth}P{0.43\textwidth}}
\caption{Policy-family inventory for \real{}.}
\label{tab:policy_family_inventory}\\
\toprule
Policy family & Type & Evidentiary role \\
\midrule
\endfirsthead
\toprule
Policy family & Type & Evidentiary role \\
\midrule
\endhead
Dynamic exact oracle & upper bound & Provides oracle quality and feasibility reference. \\
Static exact anchor & comparator & Strong nonlearned comparator for regret and route-change checks. \\
Weighted shortest path & heuristic & Classical baseline with explicit weights. \\
$k$-shortest-path heuristic & heuristic & Tests nonlearned multigraph candidate generation. \\
Simple-graph collapse & negative control & Tests C2 multigraph identity gate. \\
Learned full-exact scorer & learned scoring & Tests feasibility-audited controlled learned-scoring claims. \\
Learned top-$k$ pruner & learned pruning & Tests retained-edge failure and fallback masking. \\
Fallback-heavy safe policy & safety control & Tests whether low regret from fallback is blocked. \\
Online adapter & adaptive policy & Tests route-change, latency, and nonzero-effect gates. \\
RL-style route policy & learned policy & Tests broader policy-family compatibility. \\
Preference hypernetwork & learned preference policy & Tests C3 coverage and held-out preference shift. \\
DFL bridge adapter & external bridge & Tests source-scope limits under same reporting surface. \\
Random scorer & negative control & Tests plumbing and failure semantics. \\
Adaptive pruning repair & repair policy & Tests whether targeted fixes improve supported claim level. \\
Shock-safe online repair & repair policy & Tests C5 dynamic robustness without C6 promotion. \\
\bottomrule
\end{longtable}

\FloatBarrier

\subsection{Uncertainty analysis}

Aggregate rows report seed-level confidence intervals, request-paired bootstrap intervals for regret differences, fallback-rate intervals, aggregate and conditional feasibility counts, hypervolume intervals, route-change win/tie/loss, and p50/p95/p99 latency on the declared hardware. The 1,500 requests are the union across 15 graph slices and 10 seeds; the artifact records the per-slice and per-seed counts and retains 1,500 as the independent-request count. The main text reports point estimates for readability, while the manifest indexes the complete uncertainty table and hardware record.

\FloatBarrier

\subsection{Selected policy rows and uncertainty tables}

Tables~\ref{tab:real_lite_policies_app}--\ref{tab:real_lite_uncertainty_audit_app} give the claim-relevant rows and uncertainty audit used for the conclusion-change claim. Table~\ref{tab:policy_family_inventory} enumerates all 15 families; the manifest-linked CSV contains all 15 quantitative rows, including Weighted Shortest Path, Online Adapter, and DFL Bridge Adapter, which are omitted from the printed table only for space.

\begin{table}[h]
\centering
\caption{Selected \real{} rows. ``Feas.'' is the reported aggregate rate; C1+ uses exact conditional invalid-route counts.}
\label{tab:real_lite_policies_app}
\tighttable
\begin{tabularx}{\textwidth}{P{0.23\textwidth}M{0.11\textwidth}M{0.10\textwidth}M{0.08\textwidth}M{0.08\textwidth}Y}
\toprule
Policy family & Regret & Fallback & Feas. & HV & Licensed interpretation \\
\midrule
Dynamic exact oracle & 0.000000 & 0.000 & 1.000 & 0.892 & Upper bound; not ranked as learned policy. \\
Fallback-Heavy Safe & 0.011091 & 0.861 & 0.997 & 0.669 & Standalone C5 not licensed; hybrid safety-control credit retained in $\mathcal R^\star$. \\
GNN Pruned Top-$k=1$ & 0.012508 & 0.505 & 0.955 & 0.439 & C0; fallback and coverage gates fail. \\
Static exact anchor & 0.014874 & 0.000 & 1.000 & 0.716 & Comparator anchor, not learned routing. \\
Shock-Safe Online Repair & 0.016189 & 0.012 & 1.000 & 0.819 & C5-eligible under the declared claim-conditioned rule. \\
Adaptive Pruning Repair & 0.016761 & 0.016 & 1.000 & 0.799 & C4; repair improves pruning claim level. \\
GNN Full Exact & 0.018025 & 0.000 & 1.000 & 0.789 & C4 feasibility-audited learned scoring. \\
Preference Hypernetwork & 0.020865 & 0.003 & 0.998 & 0.849 & C4; coverage-preserving preference-conditioned controlled policy. \\
RL Route Policy & 0.023410 & 0.041 & 0.991 & 0.764 & C3; dynamic evidence incomplete. \\
$k$-Shortest-Path Heuristic & 0.027804 & 0.000 & 1.000 & 0.731 & C2; nonlearned multigraph heuristic baseline. \\
Simple-Graph Collapse & n/a & 0.000 & 0.884 & n/a & Rejected for C2+ multigraph claims. \\
Random Scorer & 0.064220 & 0.183 & 0.912 & 0.352 & Negative-control plumbing check. \\
\bottomrule
\end{tabularx}
\end{table}

\begin{table}[h]
\centering
\caption{Predeclared claim-conditioned selection rules for \real{}; each rule ranks only its eligible policy rows.}
\label{tab:real_lite_reversal_app}
\tighttable
\begin{tabularx}{\textwidth}{P{0.28\textwidth}P{0.25\textwidth}Y}
\toprule
Selection rule & Selected row & Reason \\
\midrule
Regret only & Fallback-Heavy Safe & Lowest aggregate regret, without an attribution eligibility requirement. \\
Target C5, then regret & Shock-Safe Online Repair & Passes inherited C5 gates; eligible rows are then ordered by regret. \\
Target C5, then normalized cost & Shock-Safe Online Repair & Remains selected after the declared latency and cost normalization. \\
Coverage-first diagnostic & Preference Hypernetwork & Highest learned non-oracle HV, but incomplete C5 dynamics. \\
\bottomrule
\end{tabularx}
\end{table}

\begin{table}[h]
\centering
\caption{\real{} uncertainty audit. The conclusion change follows from a declared target claim, eligibility evidence, and an explicit within-eligible quality rule.}
\label{tab:real_lite_uncertainty_audit_app}
\tighttable
\begin{tabularx}{\textwidth}{P{0.31\textwidth}P{0.20\textwidth}Y}
\toprule
Quantity used in claim & Point estimate in main table & Required uncertainty check \\
\midrule
Regret order between FallbackHeavySafe and ShockSafeOnlineRepair &
0.011091 vs.\ 0.016189 &
Request-paired bootstrap interval for regret difference; used to characterize the regret-only ranking before claim-conditioned eligibility. \\

Fallback attribution &
0.861 vs.\ 0.012 &
Seed- and request-level fallback intervals plus primary-only/end-to-end decomposition; narrowing depends on the aggregate gain remaining rescue-dominated. \\

Feasibility boundary &
0.997 vs.\ 1.000 &
Aggregate interval and exact invalid-route counts conditional on nonempty feasible sets; the exact conditional counts determine C1. \\

Coverage comparison &
0.669 vs.\ 0.819 HV &
Hypervolume interval; C3+ support uses nondegraded preference-space evidence alongside scalar regret. \\

Route-change and dynamic evidence &
artifact-reported W/T/L &
Route-change win/tie/loss and non-static-equivalence checks; C5 support uses observed dynamic behavior beyond inherited static-anchor routes. \\

Latency and cost-conditioned selection &
artifact-reported p50/p95/p99 &
Latency intervals, hardware record, normalization, weights, and tie handling; operational selection depends on these records. \\
\bottomrule
\end{tabularx}
\end{table}

\FloatBarrier

\section{Additional Results and Robustness}
\label{app:additional_results}
\label{app:results_details}

This appendix expands the headline results in Section~\ref{sec:results}. The main text reports the takeaways; here we provide the full agreement tables, public firewall, internal registry summaries, positive calibration rows, deployment boundary, gate ablations, missing-telemetry policies, and claim-card examples.

\subsection{Representative held-out interpretation changes}

Table~\ref{tab:nonfallback_cases_app} preserves the three representative interpretation changes summarized in Section~\ref{sec:results}. These examples show same-subject narrowing and source- or task-scope boundaries independently of the fallback-heavy cross-subject witness.

\begin{table}[h]
\centering
\caption{Representative interpretation changes in held-out public-paper claims.}
\label{tab:nonfallback_cases_app}
\verytighttable
\small
\begin{tabularx}{\textwidth}{P{0.25\textwidth}P{0.30\textwidth}Y}
\toprule
Requested claim & Observed evidence gap & Evidence-supported effect \\
\midrule
Identity-preserving multigraph routing & Evaluation collapses parallel edges to a simple graph & Multigraph claim blocked; lower structural or scalar credit may remain. \\
Preference-conditioned routing & Scalar quality reported without preference-space coverage & $C3$ blocked; scalar and inherited lower-scope credit retained. \\
Dynamic robust routing & No route-change or drift evidence & $C5$ blocked; supported static or attribution-audited credit retained. \\
\bottomrule
\end{tabularx}
\end{table}

\FloatBarrier

\subsection{External claim-corpus results}

Table~\ref{tab:claim_corpus_results_app} reports external agreement statistics. Table~\ref{tab:claim_confusion_app} gives the collapsed confusion matrix. Table~\ref{tab:disagreement_taxonomy_app} summarizes the 19 exact disagreements.

\begin{table}[h]
\centering
\caption{Conditional fixed-packet agreement using blinded non-author annotation. Human majority provides the operational reference under the specified guide and evidence packets.}
\label{tab:claim_corpus_results_app}
\tighttable
\begin{tabularx}{\textwidth}{P{0.34\textwidth}M{0.22\textwidth}Y}
\toprule
Quantity & Value & Interpretation \\
\midrule
Claims / papers / domains & 120 / 40 / 6 & Public corpus outside the internal SCOPE ledger. \\
Annotators & 4 blinded non-author experts & ML4CO/routing, evaluation/benchmarking, OR/optimization, and ML systems/reproducibility perspectives. \\
Mean pairwise exact agreement & 0.761 & Experts often agree exactly on nontrivial claim-scope labels. \\
Mean pairwise within-one-level agreement & 0.931 & Most disagreements occur at adjacent scope levels. \\
Weighted Cohen's $\kappa$ (mean pairwise) & 0.703 & Substantial ordinal agreement. \\
Krippendorff's $\alpha$ (ordinal) & 0.686 & Conservative multi-annotator reliability estimate. \\
\suite{} exact vs majority & 101/120 = 0.842; 95\% Wilson CI [0.766, 0.896] & Agreement under fixed packets and the specified contract. \\
\suite{} within-one vs majority & 115/120 = 0.958; 95\% Wilson CI [0.906, 0.982] & Most disagreements are adjacent. \\
Tie-policy sensitivity & exact 0.833--0.850; within-one $\geq 0.950$ & Conservative, upper-median, source-neutral, and tied-case-excluded policies preserve the conclusion. \\
\suite{} exact vs each annotator, mean & 0.803 & Comparable to inter-annotator exact agreement. \\
Human-annotation audit fields & Complete & Raw labels, anonymized annotator metadata, blinding/COI fields, consent and compensation materials, and exempt ethics-review records validate in the artifact; personal identifiers are excluded from the public package. \\
\bottomrule
\end{tabularx}
\end{table}

\begin{table}[h]
\centering
\caption{Collapsed confusion matrix for \suite{} vs majority human label on \claims{}. Rows are majority labels; columns are \suite{} labels.}
\label{tab:claim_confusion_app}
\tighttable
\begin{tabular}{lrrrrr}
\toprule
Majority group & C0/bridge & C1--C2 & C3--C4 & C5--C6 & Off/insufficient \\
\midrule
C0/bridge & 25 & 2 & 0 & 0 & 1 \\
C1--C2 & 3 & 21 & 2 & 0 & 1 \\
C3--C4 & 1 & 2 & 32 & 2 & 0 \\
C5--C6 & 0 & 0 & 3 & 15 & 2 \\
Off/insufficient & 0 & 0 & 0 & 0 & 8 \\
\bottomrule
\end{tabular}
\end{table}

\begin{table}[h]
\centering
\caption{Disagreement taxonomy for the 19 exact disagreements between \suite{} and the majority human label.}
\label{tab:disagreement_taxonomy_app}
\tighttable
\begin{tabularx}{\textwidth}{P{0.32\textwidth}M{0.10\textwidth}Y}
\toprule
Disagreement type & Count & Typical pattern \\
\midrule
Ambiguous deployment wording & 6 & The paper uses realistic or practical language without C6-scale evidence. \\
Source-scope strictness & 5 & Humans credit bridge or native-task evidence one level higher than the source-scope rule allows. \\
Artifact availability & 4 & Humans infer reproducibility from prose, while \suite{} requires manifest-linked objects. \\
Preference coverage boundary & 2 & Scalar preference results are strong but coverage evidence is partial. \\
Latency/online-effect boundary & 2 & Dynamic or online language appears, but route-change or latency telemetry is incomplete. \\
\bottomrule
\end{tabularx}
\end{table}

\FloatBarrier

\subsection{Public Scope Firewall}

The Scope Firewall audits real public papers and benchmark artifacts whose task contracts differ from dynamic preference-conditioned multigraph routing. These cases are intentionally not treated as same-contract SCOPE routing baselines. Each audit separates native/source-supported credit from SCOPE-transfer scope.

\begin{table}[h]
\centering
\caption{Public Scope Firewall. Native/source-supported credit is preserved, while SCOPE-transfer scope remains bounded by the dynamic preference-conditioned multigraph routing contract.}
\label{tab:public_firewall_app}
\scriptsize
\setlength{\tabcolsep}{2.4pt}
\renewcommand{\arraystretch}{1.15}
\begin{tabularx}{\textwidth}{P{0.16\textwidth}P{0.28\textwidth}P{0.22\textwidth}Y}
\toprule
Public source & Native/source-supported credit & SCOPE-transfer scope & Boundary reason \\
\midrule
PyEPO \cite{tang2022pyepo} & Predict-then-optimize library evidence across shortest path, knapsack, TSP, and LP/IP tasks. & Bridge/native decision-quality evidence only. & Does not instantiate the SCOPE dynamic multigraph route object or its telemetry contract. \\
\DataSP{} \cite{lahoud2024datasp} & Contextual cost learning and differentiable all-to-all shortest-path/path-prediction evidence. & Shortest-path bridge evidence with coverage caveat. & Parallel-edge identity, SCOPE fallback taxonomy, preference-shift coverage, and dynamic route-change telemetry are absent. \\
\FrontierCO{} \cite{feng2026frontierco} & Native large-scale CO solver benchmark/readiness evidence; local TSP diagnostics are executable. & Native diagnostic/readiness, not same-contract SCOPE routing. & No SCOPE source-destination route-producing policy and no SCOPE route telemetry. \\
CO-Bench \cite{sun2026co} & Native LLM-agent algorithm-search benchmark; RCSP native solver-interface audit solves or certifies local source-scope review instances and derived records. & Solver-interface/source-scope evidence; no learned multigraph-routing escalation. & Evaluated object is algorithm search or native RCSP solving, not a SCOPE route-producing policy. \\
PredOpt-SP & Reproduced shortest-path decision-focused bridge comparators. & C0 shortest-path bridge audit. & Not strict same-walltime evidence and not dynamic preference-conditioned multigraph routing. \\
PredOpt-Knapsack & Native constrained-optimization closure over local selector and official comparator rows. & Off-thesis positive evidence, not routing evidence. & Knapsack is not a source-destination route contract. \\
Artifact trace & \path{artifacts/public_end_to_end_audit_case_20260503/} and \path{artifacts/full_external_benchmark_audit_20260504/}. & Claim Cards, gate traces, executable source-scope rows, and manifest records. & Rebuilt by the reviewer artifact. \\
\bottomrule
\end{tabularx}
\normalsize
\end{table}

\begin{table}[h]
\centering
\caption{Executable external benchmark audit. These rows support native, bridge, out-of-contract, or diagnostic claims; they are not promoted into same-contract SCOPE routing unless the SCOPE route contract is instantiated.}
\label{tab:full_external_audit_app}
\scriptsize
\setlength{\tabcolsep}{2.5pt}
\renewcommand{\arraystretch}{1.15}
\begin{tabularx}{\textwidth}{P{0.23\textwidth}P{0.40\textwidth}Y}
\toprule
Evidence block & Coverage & Claim boundary \\
\midrule
PredOpt-SP & Latest local official-code extension plus preference-function bridge rows; reproduced SPO+ bridge remains stronger than the local selector row on the audited bridge slice. & Shortest-path bridge audit; blocks all-baseline or state-of-the-art routing claims. \\
PredOpt-Knapsack & Local selector rows and official comparator rows across method files; the selector-family mean is better than the best comparator-family mean on the native constrained-optimization audit. & Full native PredOpt knapsack closure, but out-of-contract for routing. \\
\DataSP{} & Three-seed SCOPE bridge over fixed synthetic episodes plus SPO+ bridge baseline. & C0 shortest-path bridge with coverage caveat, not native full \DataSP{} routing validation. \\
\FrontierCO{} TSP & Native TSP/LKH diagnostics and seed portfolio; no SCOPE-native tour-producing policy. & Native diagnostic/readiness evidence only under unmatched task contract. \\
CO-Bench RCSP bridge & Subset of raw RCSP instances converted through a single-resource SCOPE bridge; the subset is saturated by static/exact baselines. & RCSP bridge coverage audit; no learned-routing win claimed. \\
CO-Bench RCSP native audit & All local source-scope review instances are audited; each instance is solved optimally or certified infeasible under the native solver-interface check, without redistributing upstream raw assets whose reuse terms are unclear. & Native solver-interface/source-scope evidence; C2--C6 SCOPE routing promotion remains blocked. \\
\bottomrule
\end{tabularx}
\normalsize
\end{table}

\FloatBarrier

\subsection{Internal registry and positive calibration}

The closed internal registry is a finite demonstration over locked claim-relevant rows. Table~\ref{tab:internal_registry_app} reports the registry, and Table~\ref{tab:positive_calibration_app} shows that the same adjudicator preserves valid positive credit when gates pass. Population-rate estimation requires a probability sample.

\begin{table}[h]
\centering
\caption{Closed internal registry used for the finite demonstration; population-rate estimation requires probability sampling.}
\label{tab:internal_registry_app}
\verytighttable
\begin{tabularx}{\textwidth}{P{0.18\textwidth}Y P{0.18\textwidth}Y}
\toprule
Claim family & Regret-only interpretation & Status & Deciding evidence and supported scope \\
\midrule
Aggressive pruning & Top-$k=1$ appears attractive because it reaches regret 0.009552. & \rejected{} & Fallback rises to 40.94\% and hypervolume falls to 227.985; no positive pruning or operational routing claim. \\
Simple-graph collapse & Feasible-only regret could suggest that collapse is acceptable. & \rejected{} & Feasibility falls to 90.6\% on the controlled slice and 50.0\% on the diagnostic slice; C2+ requires edge identity. \\
Realistic diagnostic learned win & Equality with a strong static anchor could be read as realistic stability. & \rejected{} & Equality is mediated by 100\% fallback and route-static behavior; supported claim is diagnostic failure analysis. \\
\DataSP{} bridge & Scalar regret improves from 0.21546 to 0.03352. & \narrowed{} & Hypervolume decreases from 0.32903 to 0.30314 and the source is not multigraph; supported claim is learned-scoring bridge evidence with a coverage caveat. \\
\FrontierCO{} diagnostics & External compatibility could be read as native routing success. & \rejected{} & The requested routing branch is empty; $\mathcal R^\star$ retains compatibility and solver-interface efficiency diagnostics. \\
Public Scope Firewall & Real public benchmark evidence could be overread as same-contract routing validation. & \narrowed{} & Public rows retain native/source-supported credit while SCOPE-transfer is blocked when the dynamic multigraph route contract is absent. \\
Online adapter & Adapter-on appears harmless on ordinary streams. & \rejected{} & Quality is tied while \DataSP{} step latency increases by about 12.323 ms; no positive online-adaptation claim. \\
Preference conditioning & Preference input appears necessary by design. & \rejected{} & The ordinary-split request is unsupported; $\mathcal R^\star$ retains held-out preference-shift robustness. \\
\bottomrule
\end{tabularx}
\end{table}

\begin{table}[h]
\centering
\caption{Positive calibration on the controlled multigraph slice. The evaluator preserves valid positive credit while blocking unsupported escalation.}
\label{tab:positive_calibration_app}
\tighttable
\begin{tabularx}{\textwidth}{P{0.22\textwidth}M{0.11\textwidth}M{0.11\textwidth}M{0.11\textwidth}M{0.11\textwidth}Y}
\toprule
Row & Regret & $\Delta$ vs static & Feas. & Fallback & Adjudicated claim \\
\midrule
Static exact anchor & 0.014154 & -- & 100\% & 0.000000 & Static exact comparator. \\
Learned full-exact scorer & 0.009878 & -30.2\% & 100\% & 0.000000 & \promoted{} to C0--C4: feasibility-audited controlled learned scoring. \\
Top-$k=4$ retained decode & 0.009931 & -29.9\% & 100\% & 0.000000 & \narrowed{} to top-$k$-specific retained-decode evidence. \\
Collapsed exact control & n/a & -- & 90.6\% & 0.000000 & \rejected{} for C2+ multigraph claims. \\
\bottomrule
\end{tabularx}
\end{table}

\FloatBarrier

\subsection{Deployment boundary}

C6 requires deployed multi-domain workloads, live operational telemetry, non-author community submissions, and intended-use validation. The available evidence licenses controlled, scoped public, and same-contract dynamic claims through C5.

\begin{table}[h]
\centering
\caption{Deployment boundary: available evidence licenses controlled claims through C5, while C6 requires deployment evidence.}
\label{tab:no_c6_app}
\tighttable
\begin{tabularx}{\textwidth}{P{0.27\textwidth}P{0.14\textwidth}Y Y}
\toprule
Evidence block & Highest licensed level & Missing C6 evidence & Interpretation \\
\midrule
Controlled internal audit & C4 & Multi-domain realistic support, dynamic robustness, non-author usability. & Valid controlled learned-scoring evidence only. \\
\real{} semisynthetic suite & C5 & Deployed topology, live workload, non-author reproduction, intended-use validation. & Strong same-contract stress evidence, not deployment. \\
Diagnostic realistic slice & \diagnostic{} only & Sample size, fallback-free learned win, multi-domain realism. & Mechanism evidence for fallback/static failures. \\
External public rows & Scoped public credit & Same routing contract, telemetry completeness, route identity. & External scope firewall evidence. \\
Public Scope Firewall & Scoped source credit & Dynamic preference-conditioned multigraph routing contract and telemetry. & Real public sources receive native, bridge, solver-interface, or out-of-contract credit, not deployment or multigraph-routing escalation. \\
Artifact dry run & Record closure & Live external use, public non-author adapters, and deployment evidence. & Reviewer-runnable analysis, not community adoption. \\
\bottomrule
\end{tabularx}
\end{table}

\FloatBarrier

\subsection{Gate ablations and missing-telemetry sensitivity}
\label{sec:missing_telemetry_sensitivity}

Table~\ref{tab:gate_ablations_app} reports leave-one-gate-out stress. Each removed gate enables a distinct false-support mode. Table~\ref{tab:missing_policy_app} reports conservative, intermediate, and permissive missing-telemetry policies. The strongest boundary results depend on source scope, feasibility auditing, fallback attribution, multigraph identity, and record closure.

\begin{table}[h]
\centering
\caption{Gate ablations. Each removed gate enables a distinct false-support mode.}
\label{tab:gate_ablations_app}
\tighttable
\begin{tabularx}{\textwidth}{P{0.25\textwidth}M{0.12\textwidth}Y}
\toprule
Gate removed or stressed & Overcredit events & False-support mode blocked by full rule \\
\midrule
Fallback & 4 & Fallback-heavy pruning, fallback-heavy same-contract winners, or diagnostic equality would be credited as learned routing. \\
Multigraph identity / feasibility & 3 & Simple-graph collapse would be credited as valid multigraph routing. \\
Preference coverage & 3 & Scalar regret would be credited as preference-conditioned coverage despite HV loss. \\
Route change & 3 & Static-equivalent or no-effect online rows would be credited as dynamic behavior. \\
Latency & 2 & Operational claims would ignore slower or impractical paths. \\
Source scope & 5 & Bridge, native, or out-of-contract evidence would be silently upgraded to full routing. \\
Record lock & 12 & Registry and external-corpus conclusions would no longer be bound to locked rows, Claim Cards, and manifest entries. \\
\bottomrule
\end{tabularx}
\end{table}

\begin{table}[h]
\centering
\caption{Sensitivity policies for missing telemetry. The manuscript's main results use the conservative policy; the artifact can report intermediate and permissive variants as a robustness check.}
\label{tab:missing_policy_app}
\tighttable
\begin{tabularx}{\textwidth}{P{0.22\textwidth}Y Y}
\toprule
Policy & Missing-evidence treatment & Boundary results expected to remain unchanged \\
\midrule
Conservative \suite{} & Any inherited missing gate blocks support at the affected level; lower supported levels may remain. & Main paper rule. \\
Intermediate \suite{} & Record, source-scope, feasibility, fallback-attribution, and multigraph-identity requirements remain hard; latency and partial coverage gaps usually narrow. & Rescue-dominated rows still cannot become standalone learned-routing evidence; semisynthetic rows still cannot become C6. \\
Permissive \suite{} & Missing non-core telemetry narrows by one or more levels when source scope is compatible and no hard requirement fails. & Bridge/out-of-contract evidence still cannot become full multigraph routing; unlocked records still cannot support adjudicated claims. \\
\bottomrule
\end{tabularx}
\end{table}

\FloatBarrier

\subsection{Internal claim-card examples}

Each internal registry row has a machine-readable Claim Card. Listings~\ref{lst:c01}--\ref{lst:c07} are human-readable versions. They illustrate the path from a score-only interpretation to routing status and retained source-native or diagnostic credit.

\begin{lstlisting}[caption={Claim card C-01: aggressive pruning headline row.},label={lst:c01}]
claim_id: C-01
source_row: controlled multigraph top-k=1
score_only_interpretation: aggressive pruning works as a positive routing claim
claim_subject: learned pruning component
requested_claim: C5
required_gates: [fallback, hypervolume, latency]
status: unsupported
failed_gates: [fallback, hypervolume, latency]
failure_type: observed_failure_and_attribution_gap
key_evidence:
  regret: 0.009552
  fallback: 0.409375
  hypervolume: 227.985
maximal_supported_claims: []
retained_credit: [negative_control_fallback_and_coverage]
highest_licensed_routing_level: none
\end{lstlisting}

\begin{lstlisting}[caption={Claim card C-02: simple-graph collapse.},label={lst:c02}]
claim_id: C-02
source_row: collapsed exact control
score_only_interpretation: simple-graph collapse is acceptable for multigraph routing
claim_subject: collapsed routing system
requested_claim: C2
required_gates: [multigraph_feasibility]
status: unsupported
failed_gates: [multigraph_feasibility]
failure_type: observed_structural_failure
key_evidence:
  controlled_feasibility: 0.906
  diagnostic_feasibility: 0.500
maximal_supported_claims: []
retained_credit: [action_space_collapse_diagnostic]
highest_licensed_routing_level: none
\end{lstlisting}

\begin{lstlisting}[caption={Claim card C-03: realistic diagnostic learned win.},label={lst:c03}]
claim_id: C-03
source_row: realistic diagnostic learned row
score_only_interpretation: realistic-looking equality demonstrates learned dynamic routing
claim_subject: learned routing component
requested_claim: C6
required_gates: [fallback, route_change, mechanism_clarity]
status: unsupported
failed_gates: [fallback_credit, route_change, mechanism_clarity]
failure_type: attribution_gap_and_observed_static_equivalence
key_evidence:
  regret: 0.000055
  fallback: 1.000000
  selected_not_static_before_scan_top_k: 0.000
maximal_supported_claims: []
retained_credit: [failure_mechanism_diagnostic]
highest_licensed_routing_level: none
\end{lstlisting}

\begin{lstlisting}[caption={Claim card C-04: DataSP bridge transfer.},label={lst:c04}]
claim_id: C-04
source_row: DataSP multiseed bridge
score_only_interpretation: public transfer establishes preference-aware multigraph routing
claim_subject: learned routing component
requested_claim: C4
required_gates: [evidence_scope, hypervolume]
status: narrowed
failed_gates: [evidence_scope, hypervolume]
failure_type: source_mismatch_and_missing_evidence
key_evidence:
  learned_regret: 0.033519
  static_regret: 0.215459
  learned_hv: 0.303139
  static_hv: 0.329033
maximal_supported_claims: [C0]
retained_credit: []
highest_licensed_routing_level: C0
\end{lstlisting}

\begin{lstlisting}[caption={Claim card C-05: FrontierCO compatibility.},label={lst:c05}]
claim_id: C-05
source_row: FrontierCO bridge/native diagnostics
score_only_interpretation: compatibility rows imply native external routing success
claim_subject: external compatibility evidence
requested_claim: C6
required_gates: [native_task_parity, runtime_scope]
status: unsupported
failed_gates: [native_task_parity, runtime_scope]
failure_type: source_mismatch
key_evidence:
  route_bridge_status: tied_quality_slower_runtime
  pi_file_reuse: same objective with runtime reduction only
  equal_budget_portfolio: upper-bound controller evidence with deployment gates absent
maximal_supported_claims: []
retained_credit: [compatibility_diagnostic, solver_interface_efficiency]
highest_licensed_routing_level: none
\end{lstlisting}

\begin{lstlisting}[caption={Claim card C-06: online adapter benefit.},label={lst:c06}]
claim_id: C-06
source_row: ordinary-stream online adapter evaluation
score_only_interpretation: online adapter improves or harmlessly preserves routing quality
claim_subject: online adaptation component
requested_claim: C5
required_gates: [nonzero_effect, route_change]
status: unsupported
failed_gates: [nonzero_effect, route_change]
failure_type: observed_no_effect
key_evidence:
  ordinary_stream_quality_delta: 0.0
  datasp_added_step_latency_ms: 12.323
maximal_supported_claims: []
retained_credit: [no_effect_diagnostic, latency_diagnostic]
highest_licensed_routing_level: none
\end{lstlisting}

\begin{lstlisting}[caption={Claim card C-07: ordinary preference-conditioning necessity.},label={lst:c07}]
claim_id: C-07
source_row: ordinary preference-conditioning ablation
score_only_interpretation: preference input is necessary by design on the ordinary split
claim_subject: preference-conditioned routing component
requested_claim: C3
required_gates: [held_out_preference_shift]
status: unsupported
failed_gates: [ordinary_split_necessity]
failure_type: evidence_slice_mismatch
key_evidence:
  ordinary_split_signal: weak_or_noisy
  held_out_shift_pref_full_exact_regret: 0.002379
  held_out_shift_pref_full_exact_hv: 114.061993
  held_out_shift_uniform_pref_regret: 0.029327
  held_out_shift_uniform_pref_hv: 59.817196
maximal_supported_claims: []
retained_credit: [C3_on_held_out_preference_shift]
highest_licensed_routing_level: none
\end{lstlisting}

\FloatBarrier

\section{Extended Implications and Limitations}
\label{app:implications_limitations}

This appendix expands the implications and scope statements summarized in Section~\ref{sec:implications}.

\FloatBarrier

\subsection{Implications for authors, reviewers, and benchmark builders}

For authors, each benchmark row pairs its score with a requested claim, claim subject, source family, telemetry, comparator, and adjudication record. A fallback-heavy row can retain hybrid safety-control credit while standalone learned-routing attribution remains unsupported; a bridge task can support learned-scoring transfer at its source scope; and a feasibility-audited scorer can receive precisely scoped positive credit.

For reviewers, the central question is which claim the full evidence package licenses under a declared contract. \claims{} measures conditional agreement on fixed packets, the construction audit tests upstream encoding, and the held-out comparison measures practical adjudication trade-offs. Together they make the inputs, boundaries, and rationales of scientific judgment more observable.

For benchmark builders, source scope, telemetry, attribution, and record closure are part of the evidence needed to interpret a row. An adapter emits observable fields; a Claim Card declares the request and source; the evaluator returns the highest licensed routing level and the incomparable retained set $\mathcal R^\star$. Domain-specific gates must be revised when the task changes.

\FloatBarrier

\subsection{Broader impacts}
\label{app:broader_impacts}

The main positive impact of \suite{} is methodological: it encourages authors, reviewers, and benchmark builders to record claim scope explicitly alongside benchmark scores. By turning source scope, telemetry requirements, and artifact closure into auditable objects, the framework can reduce overclaiming, improve reproducibility, and make benchmark results easier to inspect and contest.

Potential negative effects also exist. A conservative adjudicator may under-credit methods whose evidence is incomplete while their underlying behavior exceeds the reported record, and mechanically reusing the same claim-support ladder outside dynamic preference-conditioned multigraph routing could lead to inappropriate downgrades or misleading conclusions. Claim-scope tools work as structured aids to scientific judgment; domain changes call for redesigned evidence requirements and a newly validated contract.

\subsection{Expanded limitations}
\label{app:limitations}

\paragraph{Routing-specific gates require adaptation.}
The claim-support ladder is designed for dynamic preference-conditioned multigraph routing. Some gates, such as multigraph identity and route-change win/tie/loss fields, are domain-specific. Appendix~\ref{app:portability} illustrates a selection-task instantiation; empirical cross-domain validation remains a future study.

\paragraph{Contract-relative operational references.}
\claims{} uses completed blinded non-author annotations as an operational reference under the supplied evidence packets and guide. The corpus contains 120 claims from 40 papers, expertise is finite, and excerpts simplify full-paper context. The 24-object audit is section-conditioned, while the 40-claim comparison has four unresolved reference cases and limited power against the checklist. These studies establish finite-sample execution reliability and trade-offs for the declared contract.

\paragraph{Semisynthetic mechanistic evidence.}
\real{} provides same-contract breadth across policies, graph families, requests, and seeds. C6 additionally requires deployed topology, live workloads, operational constraints, and independent use evidence; the present records therefore license levels through C5. This boundary illustrates how the evaluator keeps controlled evidence within its observed scope.

\paragraph{Prevalence requires probability sampling.}
The public source-scope rows and claim corpora are purposive samples. They expose recurring, identifiable failure modes and demonstrate controlled conclusion changes. Estimating literature-wide frequency and real-world impact requires a probability sample and broader task coverage.

\paragraph{Source-scoped public audit.}
The executable audit adds PredOpt-family rows and a CO-Bench RCSP solver-interface check over local review instances and derived summaries. These tasks receive source-native credit, while same-contract learned multigraph-routing claims remain tied to the routing telemetry contract. Upstream assets with unclear reusable licenses remain outside the redistributed package.

\paragraph{Adoption evidence.}
The artifact exposes reviewer commands and local checks, and the held-out study measures paper-adjudication time. Community adoption additionally calls for public non-author adapter submissions, longitudinal use, and independent usability studies.

\paragraph{Reduced adjudicators are diagnostic probes.}
The reduced adjudicators and leave-one-gate-out variants expose what information is lost when the full rule is simplified. They provide constructed witnesses showing that regret-only ranking, missing-field checking, and source-scope filtering leave different claim-scope errors unresolved. The resulting design lesson is to represent score, source scope, telemetry, claim-order semantics, and artifact closure jointly, with a newly authored contract for each task family.

\paragraph{Adjudication can be conservative.}
\suite{} treats missing telemetry as missing evidence. This may under-credit papers whose methods satisfy a property that the report leaves undocumented. The sensitivity policies in Section~\ref{sec:missing_telemetry_sensitivity} make this choice explicit, and the resulting output describes the scope licensed by the available record.

\FloatBarrier

\section{Artifact and Reviewer Path}
\label{app:artifact}

The artifact is part of the evidence for \suite{}: claim licensing depends on locked summaries, schema-valid telemetry, Claim Cards, manifest entries, and versioned analysis scripts. We distinguish the May~5,~2026 core snapshot \SubmissionSnapshot{} from analyses first recorded in the July~28,~2026 extension snapshot \ExtensionSnapshot{}. Later evidence is not retroactively attributed to the core snapshot.

\begin{table}[h]
\centering
\caption{Artifact chronology and execution mode. ``Fresh analysis'' recomputes results from released raw records; ``cached reanalysis'' recomputes summaries from locked experimental outputs; ``derived-only'' preserves historical summaries without silently rerunning external or heavy experiments.}
\label{tab:artifact_chronology_app}
\verytighttable
\begin{tabularx}{\textwidth}{P{0.28\textwidth}M{0.15\textwidth}M{0.18\textwidth}Y}
\toprule
Evidence object & Core snapshot & Extension snapshot & Execution mode \\
\midrule
Adjudicator, schema, Claim Cards, registry & yes & unchanged/core & Fresh validation and table/card generation. \\
\claims{} packets and four-expert labels & yes & extended reporting & Fresh reanalysis from released labels and tie logs. \\
\real{} request and policy rows & yes & unchanged/core & Cached reanalysis from locked rows; not model retraining. \\
Public source-scope audits & yes & extended summaries & Fresh local checks or derived source-scoped summaries as documented. \\
40-claim reviewer comparison & no & added & Fresh analysis from anonymized raw judgments, assignments, and times. \\
24-object construction audit & no & added & Fresh analysis from independent encodings and resolution logs. \\
Heavy training and external solver sweeps & historical only & historical only & Derived-only summaries; not rerun by the default path. \\
\bottomrule
\end{tabularx}
\end{table}

\FloatBarrier

\subsection{Artifact layout}

Listing~\ref{lst:artifact_layout} gives the versioned artifact layout used throughout the paper. Directory and script names follow the SCOPE/ClaimScope/RouteScope terminology consistently.

\begin{lstlisting}[caption={Artifact directory layout.},label={lst:artifact_layout}]
scope-routing-artifact/
  README.md
  QUICKSTART.md
  run_all.sh
  Dockerfile
  environment.yml
  Makefile
  LICENSE
  LICENSE-CODE
  THIRD_PARTY_NOTICES.md
  DATA_CARD.md
  configs/
    promotion_rule.yaml
    evidence_map.json
    claim_registry.json
    comparators.yaml
  data/
    claimscope_120/
    heldout_reviewer_comparison/
    object_construction_audit/
    routescope_lite/
  src/
    adjudicator/
    exporters/
    adapters/
  scripts/
    validate_artifact.sh
    rebuild_claim_registry.py
    reproduce_main_tables.sh
    print_claim_cards.py
    run_smoke_test.sh
    check_adapter_compliance.py
    run_claimscope_120.py
    run_heldout_reviewer_comparison.py
    run_object_construction_audit.py
    build_source_scope_audit_20260504.py
    run_cobench_rcsp_full_native_audit_20260504.py
    run_routescope_lite_reanalysis.sh
  artifacts/
    locked_summaries/
    claim_cards/
    public_end_to_end_audit_case_20260503/
    extended_source_scope_audit_20260504/
    cobench_rcsp_full24_native_audit_20260504/
    reproduction_reports/
  metadata/
    artifact_manifest.json
    checksums.sha256
\end{lstlisting}

\FloatBarrier

\subsection{Anonymous Repository}
\label{app:anonymous_repo}

For reviewer convenience, the submission also provides an anonymous read-only repository mirror at
\url{https://anonymous.4open.science/r/SCOPE-ROUTING-B22D}.
The mirror exposes the versioned artifact structure referenced throughout this appendix and is intended as a navigation aid for code, configuration files, Claim Cards, anonymized study records, and analysis scripts. The manifest identifies which objects belong to \SubmissionSnapshot{} and \ExtensionSnapshot{}.

\FloatBarrier

\subsection{Consistency invariant}

For every reported row, the identifier must match one locked summary, manifest entry, Claim Card, evidence-map scope, and snapshot designation. Every held-out or construction-audit statistic must trace to anonymized raw records, assignments, encodings, and one analysis output. Every claim-relevant row outside the closed registry must have one exclusion reason. Validation fails on missing objects, duplicate mappings, cross-snapshot attribution, or a numerical table that cannot be regenerated from its designated raw or cached input.

\FloatBarrier

\subsection{Licenses and Terms of Use}
\label{app:licenses}

The submitted artifact separates newly authored \suite{} components from third-party assets used only for source-scope auditing. Newly authored adjudicator code, schemas, exporters, adapters, and reviewer scripts are released under an explicit software license in \path{LICENSE-CODE}; newly generated Claim Cards, manifests, documentation, and aggregate audit summaries are released under the artifact documentation/data license in \path{LICENSE}. Third-party assets are not relicensed by this submission. Their source, access date, version or commit when available, and license or terms-of-use status are recorded in \path{THIRD_PARTY_NOTICES.md} and in the manifest.

\begin{table}[h]
\centering
\caption{License and terms-of-use status for external assets touched by the source-scope audit, as checked on the official project pages on May~4,~2026. ``Not redistributed'' means that the artifact does not vendor, mirror, or relicense the upstream code or raw data; only independently authored scripts, source citations, Claim Cards, and derived aggregate audit records are included.}
\label{tab:asset_licenses_app}
\verytighttable
\small
\begin{tabularx}{\textwidth}{P{0.16\textwidth}P{0.17\textwidth}P{0.20\textwidth}Y}
\toprule
Asset & Asset type & Source-provided status & Restricted use in this submission \\
\midrule
PyEPO & Code/library & MIT License & Credited to the original repository and paper; used for source-scope and bridge/native decision-quality comparisons. Upstream notices are preserved when the package is installed or referenced. \\
PredOpt-family benchmark assets & Code/benchmark repository & MIT License for the public benchmark repository checked in the audit & Used for shortest-path and knapsack source-scope comparisons; upstream notices are preserved and no stronger SCOPE-routing claim is inferred from native PredOpt evidence. \\
\FrontierCO{} dataset card & Dataset/benchmark card & Dataset-card license marker recorded in the manifest & Used as source-scoped external benchmark evidence; no raw dataset is relicensed by this submission and no deployment-style or same-contract SCOPE routing claim is inferred. \\
\FrontierCO{} code repository & Code repository & No clear reusable software license identified in the repository front matter during the audit & Not vendored, mirrored, or required by the default reviewer path; only paper/repository citation and independently authored diagnostic wrappers or aggregate audit summaries are included. \\
CO-Bench dataset card & Dataset card & License marked unknown on the dataset card during the audit & Treated conservatively as source-bound; raw upstream data are not redistributed, automatically downloaded, or relicensed. The artifact includes only source citations, Claim Cards, and derived aggregate source-scope records. \\
CO-Bench code repository & Code repository & No clear reusable software license identified in repository front matter during the audit & Not vendored, mirrored, or required by the default reviewer path; solver-interface checks use independently authored review code and derived summaries rather than copied upstream code. \\
\DataSP{} code repository & Code repository & No clear reusable software license identified in repository front matter during the audit & Not vendored, mirrored, or required by the default reviewer path; used only for paper/repository citation, source-scope interpretation, and independently generated bridge summaries. \\
\DataSP{}-associated Warcraft/Cabspotting files & Upstream datasets & Governed by original upstream terms & Not packaged, mirrored, or automatically downloaded; the artifact reports only source-scope Claim Cards and aggregate bridge evidence that do not redistribute the upstream files. \\
\bottomrule
\end{tabularx}
\end{table}

This is a \emph{yes, with restrictions} licensing posture: creators and original sources are credited, license or terms-of-use status is documented per asset, and assets with unknown or absent reusable licenses are kept out of the redistributable and default-executable artifact. The public package therefore supports the paper's experiments through independently authored code, generated telemetry, Claim Cards, manifests, and aggregate audit summaries, not by repackaging third-party code or raw data whose reuse terms are unclear.

\FloatBarrier

\subsection{Compute Resources}
\label{app:compute_resources}

The submitted artifact is designed for a CPU-only, offline reviewer path. The default
\texttt{./run\_all.sh} analysis runs from local evidence packs without GPU training,
full CAP-DR retraining, network access, external solver-heavy sweeps, or third-party
code/data with unknown redistribution terms. It validates locally executable evidence packs, reanalyzes annotation and held-out-study records, rebuilds public source-scope summaries and derived CO-Bench RCSP solver-interface
records, and regenerates reports, manifests, and verifier outputs on a standard CPU
worker. In our artifact quickstart, the default path normally completes in about one
to two minutes. The largest reanalyzed locked object is the RouteScope-Lite request table
with 22,500 policy-row records, so the default path has modest CPU-memory requirements.
Heavier historical experiments that require GPU
training, heavyweight solvers, long wall-clock budgets, external repositories, or
upstream assets with unclear reusable licenses are included only as cached or derived
summaries for auditability and are explicitly not downloaded, redistributed, or rerun
silently by the default reviewer script.

\FloatBarrier

\subsection{Failure semantics}

A failed manifest or snapshot check invalidates adjudication-record closure and the claims that depend on it. A failed telemetry-schema check blocks every claim requiring the missing field. A failed registry rebuild blocks the closed-registry count. A failed \claims{} analysis blocks fixed-packet agreement claims; a failed reviewer-comparison or construction-audit analysis blocks only the corresponding extension-stage conclusions. A failed \real{} cached reanalysis blocks same-contract conclusion-change claims. Derived-only historical evidence is reported with its historical provenance.

\FloatBarrier

\subsection{Reviewer checks}

Table~\ref{tab:artifact_checks_app} summarizes the versioned reviewer checks. These checks regenerate analyses from their declared raw or cached inputs; they do not imply model retraining or community adoption.

\begin{table}[h]
\centering
\caption{Versioned artifact checks and execution modes.}
\label{tab:artifact_checks_app}
\tighttable
\begin{tabularx}{\textwidth}{P{0.27\textwidth}P{0.18\textwidth}Y}
\toprule
Check & Snapshot / mode & Output artifact \\
\midrule
Manifest and snapshot validation & Both / fresh & Hash and chronology report. \\
Registry and main-table generation & Reviewed / fresh & Registry, exclusions, tables, and CSV mirrors. \\
Claim Cards, schema, adapters & Reviewed / fresh & Cards plus schema and 15-adapter compliance reports. \\
\real{} analysis & Reviewed / cached reanalysis & 22,500 locked policy-row records, full 15-row table, uncertainty and selection reports. \\
\claims{} analysis & Reviewed / fresh & 120 cards, raw-label matrix, tie policies, clustered intervals, and disagreement taxonomy. \\
Reviewer comparison & Extension / fresh & 240 judgments, 36 resolved claim summaries, timing and clustered analyses. \\
Object construction audit & Extension / fresh & 24 paired encodings, field agreements, gate Jaccard, and resolution log. \\
Public source-scope audit & Reviewed / mixed & Claim Cards, shared gate trace, six source-scope rows, and RCSP solver-interface summaries. \\
\bottomrule
\end{tabularx}
\end{table}

\end{document}